\documentclass{article}

\usepackage{times}
\usepackage{graphicx} 
\usepackage{natbib}
\usepackage{algorithm}
\usepackage{algorithmic}
\usepackage{amsmath,amsthm,amssymb}
\usepackage{subcaption}

\usepackage{hyperref}

\usepackage[accepted]{icml2016}

\newtheorem{thm}{Theorem}
\newtheorem*{thm*}{Theorem}
\newtheorem{lem}[thm]{Lemma}
\newtheorem*{lem*}{Lemma}
\newtheorem{prop}[thm]{Proposition}
\newtheorem*{prop*}{Proposition}

\newtheorem{defn}[thm]{Definition}

\newcommand{\X}{{\mathcal X}}
\renewcommand{\H}{{\mathcal H}}

\newcommand{\C}{{\mathcal C}}
\renewcommand{\S}{{\mathcal S}}
\newcommand{\R}{{\mathbb R}}

\newcommand{\1}{{\mathbf 1}}
\newcommand{\0}{{\mathbf 0}}
\renewcommand{\u}{{\mathbf u}}
\renewcommand{\v}{{\mathbf v}}
\newcommand{\x}{{\mathbf x}}
\newcommand{\balpha}{{\mathbf \alpha}}

\renewcommand{\>}{\rightarrow}
\newcommand{\E}{{\mathbf E}}

\newcommand{\supp}{\textup{\textrm{supp}}}

\newcommand{\curr}{\textup{\textrm{curr}}}
\newcommand{\UB}{\textup{\textrm{UB}}}
\newcommand{\lft}{\textup{\textrm{left}}}
\newcommand{\rght}{\textup{\textrm{right}}}

\renewcommand{\hat}{\widehat}

\icmltitlerunning{Mixture Proportion Estimation via Kernel Embedding of Distributions}

\begin{document} 

\twocolumn[
\icmltitle{Mixture Proportion Estimation 
via Kernel Embedding of Distributions}
\icmlauthor{Harish G. Ramaswamy${}^\dagger$}{hgramasw@in.ibm.com}
\icmladdress{IBM Research, Bangalore, India \\
					Indian Institute of Science, Bangalore, India}
\vspace{-0.6em}					
\icmlauthor{Clayton Scott}{clayscot@umich.edu} 
\vspace{-0.1em}
\icmladdress{EECS and Statistics, University of Michigan,
            Ann Arbor, USA}
 \vspace{-0.6em}					           
\icmlauthor{Ambuj Tewari}{tewaria@umich.edu}
\vspace{-0.12em}
\icmladdress{Statistics and EECS, University of Michigan,
            Ann Arbor, USA}
\icmlkeywords{mixture proportion estimation, kernel means,  }
\vskip 0.3in
]

\begin{abstract} 
Mixture proportion estimation (MPE) is the problem of estimating the weight of a component distribution in a mixture, given samples from the mixture and component. This problem constitutes a key part in many ``weakly supervised learning'' problems like learning with positive and unlabelled samples, learning with label noise, anomaly detection and crowdsourcing. While there have been several methods proposed to solve this problem, to the best of our knowledge no efficient algorithm with a proven convergence rate towards the true proportion exists for this problem. We fill this gap by constructing a provably correct algorithm for MPE, and derive convergence rates under certain assumptions on the distribution. Our method is based on embedding distributions onto an RKHS, and implementing it only requires solving a simple convex quadratic programming problem a few times. We run our algorithm on several standard classification datasets, and demonstrate that it performs comparably to or better than other algorithms on most datasets.
\end{abstract} 

\section{Introduction}
Mixture proportion estimation (MPE) is the problem of estimating the weight of a component distribution in a mixture, given samples from the mixture and component. Solving this problem happens to be a key step in solving several ``weakly supervised'' learning problems. For example, MPE is a crucial ingredient in solving the weakly supervised learning problem of learning from positive and unlabelled samples (LPUE), in which one has access to unlabelled data and positively labelled data but wishes to construct a classifier distinguishing between positive and negative data \cite{Liu+02, Denis+05, Ward+09}.  MPE also arises naturally in the task of learning a classifier with noisy labels in the training set, i.e., positive instances have a certain chance of being mislabelled as negative and vice-versa, independent of the observed feature vector \cite{LawrSchol01, BouveyGir09, StemRal09, LongServ10, Natarajan+13}. \citet{Natarajan+13} show that this problem can be solved by minimizing an appropriate cost sensitive loss. But the cost parameter depends on the label noise parameters, the computation of which can be broken into two MPE problems \cite{Scott+13}.  MPE also has applications in several other problems like anomaly rejection \cite{SandersonSc14} and crowdsourcing \cite{Raykar+10}.

When no assumptions are made on the mixture and the components, the problem is ill defined as the mixture proportion is not identifiable \cite{Scott15}. While several methods have been proposed to solve the MPE problem \cite{Blanchard+10, SandersonSc14, Scott15, ElkanNoto08, duPlessisSugi14, Jain+16}, to the best of our knowledge no provable and efficient  method is known for solving this problem in the general non-parametric setting with minimal assumptions. Some papers propose estimators that converge to the true proportion under certain conditions \cite{Blanchard+10, Scott+13, Scott15}, but they cannot be efficiently computed. Hence they use a method which is motivated based on the provable method but has no direct guarantees of convergence to the true proportion. Some papers propose an estimator that can be implemented efficiently \cite{ElkanNoto08, duPlessisSugi14}, but the resulting estimator is correct only under very restrictive conditions (see Section \ref{sec:other-algos}) on the distribution. Further, all these methods except the one by \citet{duPlessisSugi14} require an accurate binary conditional probability estimator as a sub-routine and use methods like logistic regression to achieve this. In our opinion, requiring an accurate conditional probability estimate (which is a real valued function over the instance space) for estimating the mixture proportion (a single number) is too roundabout.

Our main contribution in this paper is an efficient algorithm for mixture proportion estimation along with convergence rates of the estimate to the true proportion (under certain conditions on the distribution). The algorithm is based on embedding the distributions \cite{Gretton+12} into a reproducing kernel Hilbert space (RKHS), and only requires a simple quadratic programming solver as a sub-routine. Our method does not require the computation of a conditional probability estimate and is hence potentially better than other methods  in terms of accuracy and efficiency. We test our method on some standard datasets, compare our results against several other algorithms designed for mixture proportion estimation and find that our method performs better than or comparable to previously known algorithms on most datasets.

The rest of the paper is organised as follows. The problem set up and notations are given in Section \ref{sec:prob-setup}. In Section \ref{sec:RKHS-distance} we introduce the main object of our study, called the $\C$-distance, which essentially maps a candidate mixture proportion value to a measure of how `bad' the candidate is. We give a new condition on the mixture and component distributions that we call `separability' in Section \ref{sec:separability}, under which the $\C$-distance function explicitly reveals the true mixture proportion, and propose two estimators based on this. In Section \ref{sec:convergence-rates} we give the rates of convergence of the proposed estimators to the true mixture proportion. We give an explicit implementation of one of the estimators based on a simple binary search procedure in Section \ref{sec:grad-thres-algo}. We give brief summaries of other known algorithms for mixture proportion estimation in Section \ref{sec:other-algos} and list their characteristics and shortcomings. We give details of our experiments in Section \ref{sec:expts} and conclude in Section \ref{sec:concl}.

\section{Problem Setup and Notations}
\label{sec:prob-setup}
\label{submission}
Let $G,H$ be distributions over a compact metric space $\X$ with supports given by $\supp(G),\supp(H)$. Let $\kappa^*\in [0,1)$ and let $F$ be a distribution that is given by a convex combination (or equivalently, a mixture) of $G$ and $H$ as follows:
\[
F= (1-\kappa^*) G + \kappa^* H .
\]
Equivalently, we can write
\[
G= (\lambda^*) F + (1-\lambda^*) H ,
\]
where $\lambda^*=\frac{1}{1-\kappa^*}$.  Given samples $\{x_1,x_2,\ldots,x_n\}$ drawn i.i.d. from $F$ and $\{x_{n+1},\ldots,x_{n+m}\}$ drawn i.i.d. from $H$, the objective in mixture proportion estimation (MPE) \cite{Scott15} is to estimate $\kappa^*$. 

Let $\H$ be a reproducing kernel Hilbert space (RKHS) \cite{Aronszajn50, BerlinetThomas04} with a positive semi-definite kernel $k:\X\times\X\>\R$. Let $\phi:\X\>\H$ represent the kernel mapping $x\mapsto k(x,.)$. For any distribution $P$ over $\X$, let $\phi(P)=\E_{X\sim P} \phi(X)$. It can be seen that for any distribution $P$ and $f\in\H$, that $\E_{X\sim P} f(X) = \langle f,\phi(P)\rangle_\H$. Let $\Delta_{n+m}\subseteq\R^{n+m}$ be the $(n+m-1)$-dimensional probability simplex given by $\Delta_{n+m}=\{\mathbf p\in[0,1]^{n+m}: \sum_i p_i=1\}$. Let $\C,\C_S$ be defined as
\begin{align*}
\C & =\{w \in \H: w=\phi(P), \text{for some distribution } P \}, \\
\C_S & =\{w \in \H: w=\sum_{i=1}^{n+m} \alpha_i \phi(x_i), \text{ for some } \balpha\in\Delta_{n+m} \}.
\end{align*}
Clearly, $\C_S\subseteq \C$, and both $\C,\C_S$ are convex sets. 

Let $\hat{F}$ be the distribution over $\X$ that is uniform over $\{x_1,x_2,\ldots,x_n\}$.  Let $\hat{H}$ be the distribution over $\X$ that is uniform over $\{x_{n+1},\ldots,x_{n+m}\}$. As $F$ is a mixture of $G$ and $H$, we have that some $S_1\subseteq \{x_1,\ldots,x_n\}$ is drawn from $G$ and the rest from $H$. We let $\hat{G}$ denote the uniform distribution over $S_1$. On average, we expect the cardinality of $S_1$ to be $\frac{n}{\lambda^*}$. Note that we do not know $S_1$ and hence cannot compute $\phi(\hat{G})$ directly, however we have that $\phi(\hat{G})\in\C_S$.

\section{RKHS Distance to Valid Distributions}
\label{sec:RKHS-distance}
Define the ``$\C$-distance'' function $d:[0,\infty)\>[0,\infty)$ as follows:
\begin{align}
d(\lambda) = \inf_{w\in\C} \| \lambda \phi(F) + (1-\lambda)\phi(H) - w \|_\H .
\label{eqn:d-defn}
\end{align}
Intuitively, $d(\lambda)$  reconstructs $\phi(G)$ from $F$ and $H$ assuming $\lambda^*=\lambda$, and computes its distance to $\C$.  Also, define the empirical version of the $\C$-distance function, $\hat d:[0,\infty)\>[0,\infty)$, which we call the $\C_S$-distance function, as
\begin{align}
\hat d(\lambda) = \inf_{w\in\C_S} \| \lambda \phi(\hat{F}) + (1-\lambda)\phi(\hat H) - w \|_\H \,.
\label{eqn:d-hat-defn}
\end{align}

Note that the $\C_S$-distance function $\hat d(\lambda)$ can be computed efficiently via solving a quadratic program.  For any $\lambda\geq0$, let $\u_\lambda\in\R^{n+m}$ be such that $\u_\lambda^\top = \frac{\lambda}{n} ([\1_n^\top, \0_m^\top]) + \frac{1-\lambda}{m} ([\0_n^\top, \1_m^\top])$, where $\1_n$ is the $n$-dimensional all ones vector, and $\0_m$ is the $m$-dimensional all zeros vector. Let $K\in\R^{(n+m)\times(n+m)}$ be the kernel matrix given by $K_{i,j}=k(x_i,x_j)$. We then have 
\[
(\hat d(\lambda))^2=  \inf_{\v\in\Delta_{n+m}} (\u_\lambda - \v)^\top K (\u_\lambda - \v)\,.
\] 

We now give some basic properties of the $\C$-distance function and the $\C_S$-distance function that will be of use later. All proofs not found in the paper can be found in the supplementary material.
\begin{prop}
\label{prop:d-is-zero}
\begin{align*}
d(\lambda) &=0,   \qquad \forall \lambda\in[0,\lambda^*], \\
\hat d(\lambda) &=0,   \qquad \forall \lambda\in[0,1]\,.
\end{align*} 
\end{prop}

\begin{prop}
\label{prop:d-inc-conv}
 $d(.)$ and $\hat d(.)$ are non-decreasing convex functions on $[0,\infty)$.
\end{prop}

Below, we give a simple reformulation of the $\C$-distance function and basic lower and upper bounds that reveal its structure.

\begin{prop}
\label{prop-d-reformulation}
For all $\mu\geq 0$,
\begin{align*}
d(\lambda^*+\mu)	&= \inf_{w\in\C} \| \phi(G) + \mu(\phi(F) - \phi(H)) - w \|_\H\,.
\end{align*}
\end{prop}
\begin{prop}
\label{prop:basic-bounds-d}
For all $\lambda,\mu \geq 0$,
\begin{align}
d(\lambda) &\geq \lambda \|\phi(F) - \phi(H)\| - \sup_{w\in\C} \|\phi(H)-w\|,  \label{eqn:bound-d-1}\\
d(\lambda^*+ \mu) &\leq  \mu \|\phi(F) - \phi(H)\| \,.\label{eqn:bound-d-4}
\end{align}
\end{prop}

Using standard results of  \citet{Smola+07}, we can show that the kernel mean embeddings of the empirical versions of $F$, $H$ and $G$ are close to the embeddings of the distributions themselves.

\begin{lem}
\label{prop:empirical-close-to-truth}
Let the kernel $k$ be such that $k(x,x)\leq 1$ for all $x\in\X$. Let $\delta\in(0,1/4]$. The following holds with probability $1-4\delta$ (over the sample $x_1,\ldots,x_{n+m}$) if $n>2(\lambda^*)^2 \log\left(\frac{1}{\delta}\right)$,
\begin{align*}
\|\phi(F) - \phi(\hat{F})\|_\H &\leq 
\frac{3\sqrt{\log(1/\delta)}}{\sqrt{n}}\\
\|\phi(H) - \phi(\hat{H})\|_\H &\leq
\frac{3\sqrt{\log(1/\delta)}}{\sqrt{m}}  \\
\|\phi(G) - \phi(\hat{G})\|_\H  &\leq
\frac{3\sqrt{\log(1/\delta)}}{\sqrt{n/(2\lambda^*)}}.
\end{align*}
\end{lem}

We will call this $1-4\delta$ high probability event as $E_\delta$. All our results hold under this event. 

Using Lemma \ref{prop:empirical-close-to-truth} one can show that the $\C$-distance function and the $\C_S$-distance function are close to each other. Of particular use to us is an upper bound on the $\C_S$-distance function $\hat d(\lambda)$ for $\lambda\in[1,\lambda^*]$, and a general lower bound on $\hat d(\lambda) - d(\lambda)$.
\begin{lem}
\label{prop:d_hat_upper}
Let $k(x,x)\leq 1$ for all $x\in\X$. Assume $E_\delta$. For all $\lambda\in[1,\lambda^*]$ we have that
\[
\hat d(\lambda) \leq 
\left(2-\frac{1}{\lambda^*}+\frac{\sqrt{2}}{\sqrt{\lambda^*}}\right)\lambda\cdot \frac{3\sqrt{\log(1/\delta)}}{\sqrt{\min(m,n)}}.
\]
\end{lem}

\begin{lem}
\label{prop:d_hat_lower}
Let $k(x,x)\leq 1$ for all $x\in\X$. Assume $E_\delta$. For all $\lambda\geq 1$, we have
\[
\hat d(\lambda) \geq 
d(\lambda) - (2\lambda-1)\cdot \frac{3\sqrt{\log(1/\delta)}}{\sqrt{\min(m,n)}}\,.
\]
\end{lem}

\section{Mixture Proportion Estimation under a Separability Condition}
\label{sec:separability}
\citet{Blanchard+10, Scott15} observe that without any assumptions on $F,G$ and $H$, the mixture proportion $\kappa^*$ is not identifiable, and postulate an ``irreducibility'' assumption under which $\kappa^*$ becomes identifiable. The irreducibility assumption  essentially states that $G$ cannot be expressed as a non-trivial mixture of $H$ and some other distribution. \citet{Scott15} propose a stronger assumption than irreducibility under which they provide convergence rates of the estimator proposed by \citet{Blanchard+10} to the true mixture proportion $\kappa^*$. We call this condition as the ``anchor set'' condition as it is similar to the ``anchor words'' condition of \citet{Arora+12} when the domain $\X$ is finite. 

\begin{defn}  A family of subsets $\S \subseteq 2^\X$,  and distributions $G, H$ are said to satisfy the \emph{anchor set condition with margin $\gamma>0$}, if  there exists a compact set $ A\in\S$ such that $A\subseteq \supp(H) \setminus \supp(G)$ and $H(A)\geq \gamma$.
\end{defn}

We propose another condition which is similar to the anchor set condition (and is defined for a class of functions on $\X$ rather than subsets of $\X$). Under this condition we show that the $\C$-distance function (and hence the $\C_S$-distance function) reveals the true mixing proportion $\lambda^*$.

\begin{defn} A class of functions $\H\subseteq\R^\X$, and distributions $G, H$ are said to satisfy \emph{separability condition with margin $\alpha>0$ and tolerance $\beta$}, if $\exists h\in\H, \|h\|_\H\leq 1$  and
\begin{equation*}
 \E_{X\sim G} h(X) \leq \inf_{x} h(x) + \beta \leq \E_{X\sim H} h(X) - \alpha\,.
\end{equation*}
We say that a kernel $k$ and distributions $G,H$ satisfy the separability condition, if the unit norm ball in its RKHS and distributions $G,H$ satisfy the separability condition.
\end{defn}

Given a family of subsets satisfying the anchor set condition with margin $\gamma$, it can be easily seen that the family of functions given by the indicator functions of the family of subsets satisfy the separability condition with margin $\alpha=\gamma$ and tolerance $\beta=0$. Hence this represents a natural extension of the anchor set condition to a function space setting. 

Under separability one can show that $\lambda^*$ is the ``departure point from zero'' for the $\C$-distance function.

\begin{thm}
\label{thm:d_lower}
Let the kernel $k$, and distributions $G, H$ satisfy the separability condition with margin $\alpha>0$ and tolerance $\beta$. Then $\forall \mu>0$
\[
d(\lambda^*+\mu) \geq \frac{\alpha \mu}{\lambda^*} - \beta\,.
\]
\end{thm}
\begin{proof}(Sketch)
For any inner product $\langle . , . \rangle$ and its norm $\|.\|$ over the vector space $\H$, we have that $\| f \| \geq \langle f, g \rangle$ for all $g\in \H$ with $\|g\| \leq 1$. 
The proof mainly follows by lower bounding the norm in the definition of $d(.)$, with an inner product with the witness $g$ of the separability condition. 
\end{proof}
Further, one can link the separability condition and the anchor set condition via universal kernels (like the Gaussian RBF kernel) \cite{Michelli+06}, which are kernels whose RKHS is dense in the space of all continuous functions over a compact domain. 

\begin{thm}
\label{thm:kernel-separability-universality-anchor-set}
Let the kernel $k:\X\times\X\>[0,\infty)$ be universal. Let the distributions $G,H$ be such that they satisfy the anchor set condition with margin $\gamma>0$ for some family of subsets of $\X$. Then, for all $\theta>0$, there exists a $\beta>0$ such that the kernel $k$, and distributions $G,H$ satisfy the separability condition with margin $\beta \theta$ and tolerance $\beta$.
\end{thm}
\begin{proof} (Sketch)
As the distributions $G, H$ satisfy the anchor set condition, there must exist a continuous non-negative function that is zero on the support of $G$ and greater than one on the set $A$ that witnesses the anchor set condition. Due to universality of the kernel $k$, there must exist an element in its RKHS that arbitrarily approximates this function. The normalised version of this function forms a witness to the separability condition.
\end{proof}
The ultimate objective in mixture proportion estimation is to estimate $\kappa^*$ (or equivalently $\lambda^*$). If one has direct access to $d(.)$ and the kernel $k$ and  distributions $G,H$ satisfy the separability condition with tolerance $\beta=0$, then we have by Proposition \ref{prop:d-is-zero} and Theorem \ref{thm:d_lower} that
\[
\lambda^* = \inf\{\lambda:d(\lambda)>0\}.
\]
We do not have direct access to $d(.)$, but we can calculate $\hat d(.)$. From Lemmas \ref{prop:d-is-zero} and \ref{prop:d_hat_lower}, we have that for all $\lambda\in[0,\lambda^*]$, $\hat d(\lambda)$ converges to $0$ as the sample size $\min(m,n)$ increases. From Lemma \ref{prop:d_hat_lower} we have that for all $\lambda\geq 0$, $\hat d(\lambda) \geq d(\lambda) - \epsilon$ for any $\epsilon>0$ if $\min(m,n)$ is large enough. Hence $\hat d(.)$ is a good surrogate for $d(.)$ and based on this observation we propose two strategies of estimating $\lambda^*$ and show that the errors of both these strategies can be made to approach $0$ under the separability condition. 

The first estimator is called the value thresholding estimator. For some $\tau\in[0,\infty)$ it is defined as,
\[
\hat \lambda^V_\tau = \inf\{ \lambda: \hat d(\lambda)\geq\tau \}\,.
\]

The second estimator is called the gradient thresholding estimator. For some $\nu\in[0,\infty)$ it is defined as
\[
\hat \lambda^G_\nu = \inf\{ \lambda: \exists g \in \partial \hat d(\lambda), g\geq\nu \},
\] 
where $\partial \hat d(\lambda)$ is the sub-differential of $\hat d(.)$ at $\lambda$. As $\hat d(.)$  is a convex function, the slope of $\hat d(.)$ is a non-decreasing function and thus thresholding the gradient is also a viable strategy for estimating $\lambda^*$. 

To illustrate some of the ideas above, we plot $\hat d(.)$ and $\nabla \hat d(.)$ for two different true mixing proportions $\kappa^*$ and sample sizes in Figure \ref{fig:plots-illus}. The data points from the component and mixture distribution used for computing the plot are taken from the \texttt{waveform} dataset.

\section{Convergence of Value and Gradient Thresholding Estimators}
\label{sec:convergence-rates}
\begin{figure*}[t]
\begin{center}
\begin{subfigure}[b]{0.48\textwidth}
\includegraphics[width=\textwidth]{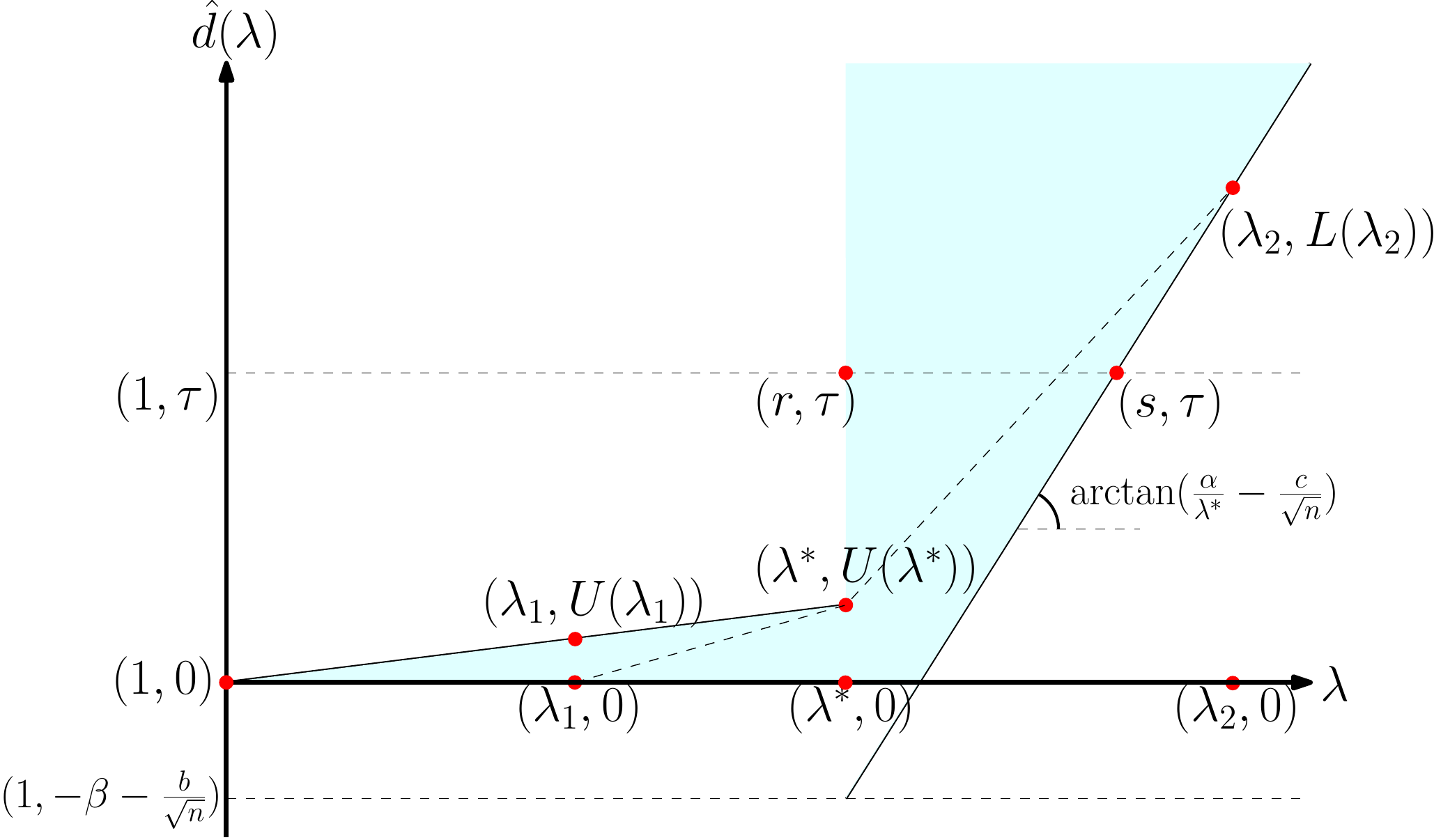}
\caption{The feasible pairs of $(\lambda,\hat d(\lambda))$  is shaded in light cyan.}
\label{subfig:distance-lambda} 
\end{subfigure}
\hspace{1em}
\begin{subfigure}[b]{0.48\textwidth}
\includegraphics[width=\textwidth]{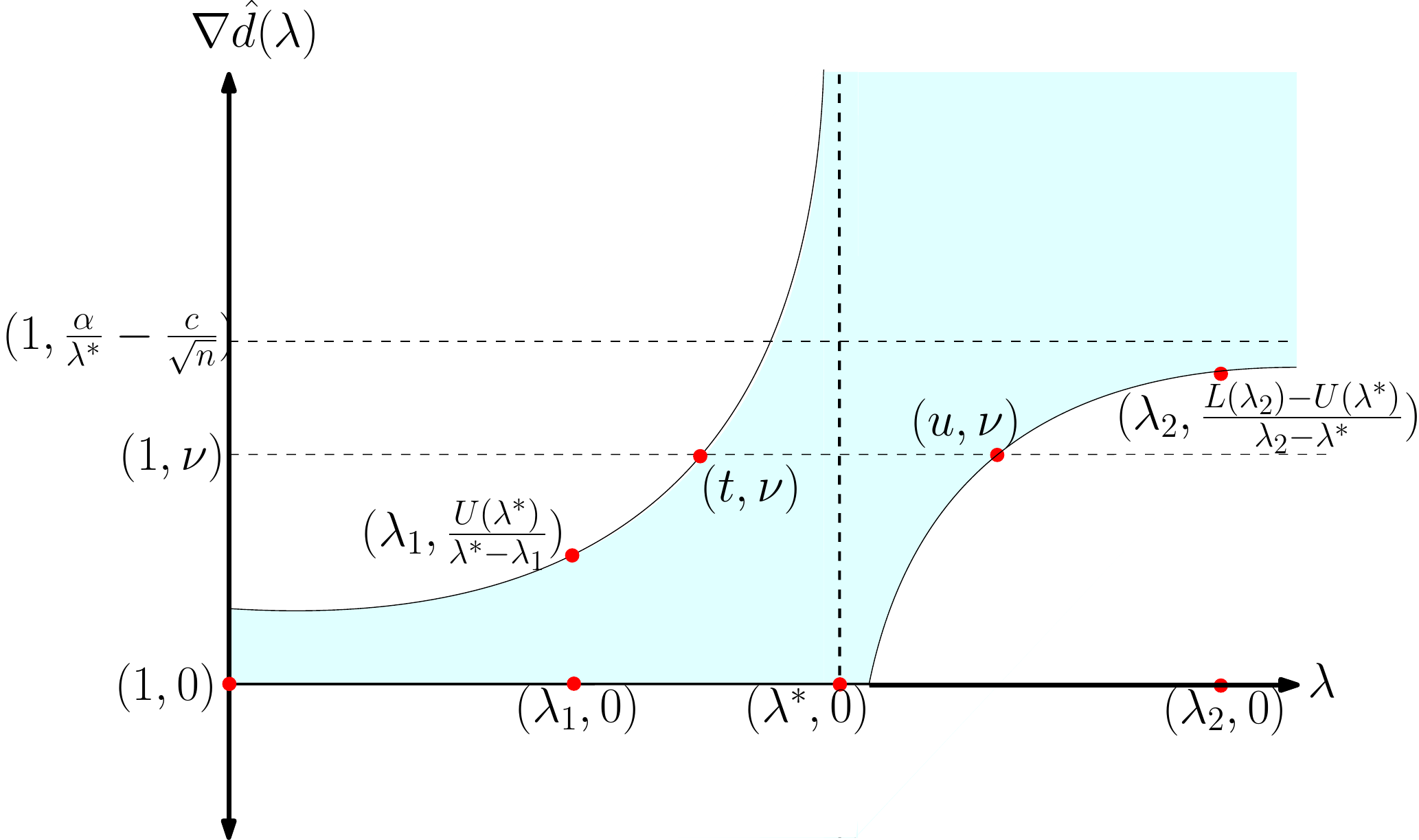}
\caption{The feasible pairs of $(\lambda,\nabla\hat d(\lambda))$ is shaded in light cyan.}
\label{subfig:grad-distance-lambda}
\end{subfigure}
\end{center}
\caption{Illustration of the upper and lower bounds on $\hat d(\lambda)$ and $\nabla\hat d(\lambda)$, under separability conditions (with margin $\alpha$ and tolerance $\beta$) and event $E_\delta$.}
\label{fig:proof-by-picture}
\end{figure*}

We now show that both the value thresholding estimator $\hat\lambda^V_\tau$ and the gradient thresholding estimator $\hat\lambda^G_\nu$ converge to $\lambda^*$ under appropriate conditions.

\begin{thm}
\label{thm:val-thresh-error}
Let $\delta\in(0,\frac{1}{4}]$. Let $k(x,x)\leq 1$ for all $x\in\X$. Let the kernel $k$, and distributions $G, H$ satisfy the separability condition with tolerance $\beta$ and margin $\alpha>0$. Let the number of samples be large enough such that  $\min(m,n)>\frac{(12  \cdot \lambda^*)^2\log(1/\delta)}{\alpha^2}$. Let the threshold $\tau$ be such that $\frac{3\lambda^*\sqrt{\log(1/\delta)}(2-1/\lambda^*+\sqrt{2/\lambda^*})}{\sqrt{\min(m,n)}} \leq \tau \leq \frac{6\lambda^*\sqrt{\log(1/\delta)}(2-1/\lambda^*+\sqrt{2/\lambda^*})}{\sqrt{\min(m,n)}} $. We then have with probability $1-4\delta$
\begin{align*}
   \lambda^* - \hat \lambda^V_\tau 
   &\leq  0 ,\\
     \hat \lambda^V_\tau -   \lambda^* 
   &\leq  \frac{\beta\lambda^*}{\alpha} + c  \cdot \sqrt{\log(1/\delta)} \cdot (\min(m,n))^{-1/2} ,
\end{align*}
where $c=\left( \frac{6\alpha (\lambda^*)^2(2-1/\lambda^*+\sqrt{2/\lambda^*})+  2\lambda^*(3\alpha+6\lambda^*(2+\alpha+\beta))  }{\alpha^2} \right)$.
\end{thm}
\begin{proof}(Sketch)
Under event $E_\delta$, Lemma \ref{prop:d_hat_upper} gives an upper bound on $\hat d(\lambda)$ for $\lambda\in[1,\lambda^*]$, which is denoted by the line $(\lambda,U(\lambda))$ in Figure \ref{subfig:distance-lambda}. Under event $E_\delta$ and the separability condition, Lemma \ref{prop:d_hat_lower} and Theorem \ref{thm:d_lower} give a lower bound on $\hat d(\lambda)$ for $\lambda\geq\lambda^*$ and is denoted by the line $(\lambda,L(\lambda))$ in Figure \ref{subfig:distance-lambda}. These two bounds immediately give upper and lower bounds on the value thresholding estimator $\hat\lambda^V_\tau$ for any $\tau\in[0,\infty)$. An illustration is provided in Figure \ref{subfig:distance-lambda} by the horizontal line through $(1,\tau)$. The points of intersection of this line with the feasible values of $(\lambda,\hat d(\lambda))$ as  in Figure \ref{subfig:distance-lambda}, given by $r$ and $s$ in the figure form lower and upper bounds respectively for $\hat\lambda^V_\tau$.

\end{proof}

\begin{thm}
\label{thm:grad-thresh-error}
Let $k(x,x)\leq 1$ for all $x\in\X$. Let the kernel $k$, and distributions $G, H$ satisfy the separability condition with tolerance $\beta$ and margin $\alpha>0$. Let $\nu\in[\frac{\alpha}{4\lambda^*},  \frac{3 \alpha}{4 \lambda^*}]$ and $\sqrt{\min(m,n)}\geq\frac{36 \sqrt{\log(1/\delta)}}{\frac{\alpha}{\lambda^*}-\nu}$. We then have with probability $1-4\delta$
\begin{align*}
   \lambda^* - \hat \lambda^G_\nu 
   &\leq  c  \cdot \sqrt{\log(1/\delta)} \cdot (\min(m,n))^{-1/2} ,\\
     \hat \lambda^G_\nu -  \lambda^* 
   &\leq  \frac{4\beta\lambda^*}{\alpha} + c'  \cdot \sqrt{\log(1/\delta)} \cdot (\min(m,n))^{-1/2} ,
\end{align*}
for constants $c=(2\lambda^*-1+\sqrt{2\lambda^*}) \cdot \frac{12\lambda^*}{\alpha}$ and $c'=\frac{144 (\lambda^*)^2(\alpha+4\beta)}{\alpha^2}$.
\end{thm}
\begin{proof}(Sketch)
The upper and lower bounds on $\hat d(\lambda)$ given by Lemmas \ref{prop:d_hat_lower}, \ref{prop:d_hat_upper} and Theorem \ref{thm:d_lower} also immediately translate into upper and lower bounds for $\nabla \hat d(\lambda)$ (assume differentiability of $\hat d(.)$ for convenience) due to convexity of $\hat d(.)$. As shown in Figure \ref{subfig:distance-lambda}, the gradient of $\hat d(.)$ at some $\lambda_1<\lambda^*$ is upper bounded by the slope of the line joining $(\lambda_1,0)$ and $(\lambda^*,U(\lambda^*))$. Similarly, the gradient of $\hat d(.)$ at some $\lambda_2>\lambda^*$ is lower bounded by the slope of the line joining $(\lambda^*,U(\lambda^*))$ and $(\lambda_2,L(\lambda_2))$.
Along with trivial bounds on $\nabla \hat d(\lambda)$, these bounds give the set of feasible values for the ordered pair $(\lambda, \nabla \hat d(\lambda))$, as illustrated in Figure \ref{subfig:grad-distance-lambda}. This immediately gives bounds on $\hat\lambda^G_\nu$ for any $\nu\in[0,\infty)$. An illustration is provided in Figure \ref{subfig:grad-distance-lambda} by the horizontal line through $(1,\nu)$. The points of intersection of this line with the feasible values of $(\lambda,\nabla\hat d(\lambda))$ as in  Figure \ref{subfig:grad-distance-lambda}, given by $t$ and $u$ in the figure form lower and upper bounds respectively for $\hat\lambda^G_\nu$.	
\end{proof}
\textbf{Remark:} Both the value and gradient thresholding estimates converge to $\lambda^*$ with rates $O(m^{-\frac{1}{2}})$, if the kernel satisfies the separability condition with a tolerance $\beta=0$. In the event of the kernel only satisfying the separability condition with tolerance $\beta>0$, the estimates converge to within an additive factor of $\frac{\beta\lambda^*}{\alpha}$. As shown in Theorem \ref{thm:kernel-separability-universality-anchor-set}, with a universal kernel the ratio $\frac{\beta}{\alpha}$ can be made arbitrarily low, and hence both the estimates actually converge to $\lambda^*$, but a specific rate is not possible, due to the dependence of the constants on $\alpha$ and $\beta$, without further assumptions on $G$ and $H$.

\section{The Gradient Thresholding Algorithm}
\label{sec:grad-thres-algo}
As can be seen in Theorems \ref{thm:val-thresh-error} and \ref{thm:grad-thresh-error}, the value and gradient thresholding estimators both converge to $\lambda^*$ at a rate of $O(m^{-\frac{1}{2}})$, in the scenario where we know the optimal threshold. In practice, one needs to set the threshold heuristically, and we observe that the  estimate $\hat\lambda^V_\tau$ is much more sensitive to the threshold $\tau$, than the gradient thresholding estimate $\hat\lambda^G_\nu$ is to the threshold $\nu$. This agrees with our intuition of the asymptotic behavior of $\hat d(\lambda)$ and $\nabla\hat d(\lambda)$ -- the curve of $\hat d(\lambda)$ vs $\lambda$ is close to a hinge, whereas the curve of $\nabla\hat d(\lambda)$ vs $\lambda$ is close to a step function. This can also be seen in Figure \ref{subfig:3200_75}. Hence, our estimator of choice is the gradient thresholding estimator and we give an algorithm for implementing it in this section. 

Due to the convexity of $\hat d(.)$, the slope $\nabla \hat d(.)$ is an increasing function, and thus the gradient thresholding estimator $\hat\lambda^G_\nu$ can be computed efficiently via binary search. The details of the computation are  given in Algorithm \ref{alg:GT_KM}.

Algorithm \ref{alg:GT_KM} maintains upper and lower bounds ($\lambda_\lft$ and $\lambda_\rght$) on the gradient thresholding estimator,\footnote{We assume an initial upper bound of 10 for convenience, as we don't gain much by searching over higher values. $\hat\lambda^G_\nu=10$ corresponds to a mixture proportion estimate of $\hat\kappa=0.9$.} estimates the slope at the current point $\lambda_\curr$ and adjusts the upper and lower bounds based on the computed slope. The slope at the current point $\lambda_\curr$ is estimated numerically by computing the value of $\hat d(.)$ at $\lambda_\curr \pm \frac{\epsilon}{4}$ (lines 9 to 15). We compute the value of $\hat d(\lambda)$ for some given $\lambda$ using the general purpose convex programming solver CVXOPT.
\footnote{The accuracy parameter $\epsilon$ must be set large enough so that the optimization error in computing $\hat d(\lambda_\curr \pm \frac{\epsilon}{4})$ is small when compared to $\hat d(\lambda_\curr + \frac{\epsilon}{4}) - \hat d(\lambda_\curr - \frac{\epsilon}{4})$.}

We employ the following simple strategy for model selection (choosing the kernel $k$ and threshold $\nu$). Given a set of kernels, we choose the kernel for which the empirical RKHS distance between the distributions $F$ and $H$, given by $\|\phi(\hat F) - \phi(\hat H) \|_\H$ is maximized. This corresponds to choosing a kernel for which the ``roof'' of the step-like function $\nabla\hat d(.)$ is highest. We follow two different strategies for setting the gradient threshold $\nu$. One strategy is motivated by Lemma \ref{prop:d_hat_upper}, where we can see that the slope of $\hat d(\lambda)$ for $\lambda\in[1,\lambda^*]$ is $O(1/\sqrt{\min(m,n)})$ and based on this we set $\nu=1/\sqrt{\min(m,n)}$. The other strategy is based on empirical observation, and is set as a convex combination of the initial slope of $\hat d$ at $\lambda=1$ and the final slope at $\lambda=\infty$ which is equal to the RKHS distance  between the distributions $F$ and $H$, given by $\|\phi(\hat F) - \phi(\hat H) \|_\H$. We call the resulting two algorithms as ``KM1'' and ``KM2'' respectively in our experiments.\footnote{In  KM2,  $\nu=0.8* \text{init\_slope} + 0.2*\text{final\_slope}$}

\begin{algorithm}[t]
 \caption{Kernel mean based gradient thresholder} 
 \label{alg:GT_KM}
 \begin{algorithmic}[1]
   \STATE {\bfseries Input: }$\x_1,\x_2,\ldots,\x_n$ drawn from mixture $F$ and $\x_{n+1},\ldots,\x_{n+m}$ drawn from component $H$
   \STATE {\bfseries Parameters:} $k:\X\times\X\>[0,\infty)$, $\nu\in[0,\infty)$
   \STATE {\bfseries Output: } $\hat \lambda^G_\nu$
   \STATE {\bfseries Constants: } $\epsilon=0.04, \lambda_\UB=10$
   \STATE $\lambda_\lft=1, \lambda_\rght=\lambda_\UB$
   \STATE $K_{i,j} = k(\x_i,\x_j)$ for $1\leq i,j \leq n+m$
   \STATE \textbf{while } $\lambda_\rght - \lambda_\lft \geq \epsilon$
   \STATE ~~~~~~~~$\lambda_\curr=\frac{\lambda_\rght+\lambda_\lft}{2}$
   \STATE ~~~~~~~~$\lambda_1=\lambda_\curr - \epsilon/4$
   \STATE ~~~~~~~~$\u_1 = \frac{\lambda_1}{n} ([\1_n^\top, \0_m^\top]) + \frac{1-\lambda_1}{m} ([\0_n^\top, \1_m^\top])$
   \STATE ~~~~~~~~$d_1=\hat d(\lambda_1)^2= \displaystyle\min_{\v\in\Delta_{n+m}} (\u_1 -\v)^\top K (\u_1 - \v)$
   \STATE ~~~~~~~~$\lambda_2=\lambda_\curr + \epsilon/4$
   \STATE ~~~~~~~~$\u_2 = \frac{\lambda_2}{n} ([\1_n^\top, \0_m^\top]) + \frac{1-\lambda_2}{m} ([\0_n^\top, \1_m^\top])$
   \STATE ~~~~~~~~$d_2=\hat d(\lambda_2)^2=\displaystyle\min_{\v\in\Delta_{n+m}} (\u_2 -\v)^\top K (\u_2 - \v)$
   \STATE ~~~~~~~~$s=\frac{\sqrt{d_2} - \sqrt{d_1}}{\lambda_2 - \lambda_1}$
   \STATE ~~~~~~~~\textbf{if } $s>\nu$:
   \STATE ~~~~~~~~~~~~~~~~$\lambda_\rght=\lambda_\curr$
   \STATE ~~~~~~~\textbf{else}:
   \STATE ~~~~~~~~~~~~~~~~$\lambda_\lft=\lambda_\curr$
   \STATE \textbf{return }$\lambda_\curr$
 \end{algorithmic}
\end{algorithm}

\section{Other Methods for Mixture Proportion Estimation}
\label{sec:other-algos}
\citet{Blanchard+10} propose an estimator based on the following equality, which holds under an irreducibility condition (which is a strictly weaker requirement than the anchor set condition),
$
 \kappa^*=\inf_{\S \in \Theta, H(\S)>0} \frac{F(\S)}{H(\S)} ,
$
where $\Theta$ is the set of measurable sets in $\X$. The estimator proposed replaces the exact terms $F(\S)$ and $H(\S)$ in the above ratio with the empirical quantities $\hat F(\S)$ and $\hat H(\S)$ and includes VC-inequality  based correction terms in the numerator and denominator and restricts $\Theta$ to a sequence of VC classes. \citet{Blanchard+10} show that the proposed estimator converges to the true proportion under the irreducibility condition and also show that the convergence can be arbitrarily slow. Note that the requirement of taking infimum over VC classes makes a direct implementation of this estimator computationally infeasible.

\citet{Scott15} show that the estimator of \citet{Blanchard+10} converges to the true proportion at the rate of $1/\sqrt{\min(m,n)}$ under the anchor set condition, and also make the observation that the infimum over the sequence of VC classes can be replaced by an infimum over just the collection of base sets (e.g. the set of all open balls). Computationally, this observation reduces the complexity of a direct implementation of the estimator to $O(N^d)$ where $N=m+n$ is the number of data points, and $d$ is the data dimension. But the estimator still remains intractable for even datasets with moderately large number of features. 

\citet{SandersonSc14, Scott15} propose algorithms based on the estimator of \citet{Blanchard+10}, which treats samples from $F$ and samples from $H$ as positive and negative classes, builds a conditional probability estimator and computes the estimate of $\kappa^*$ from the constructed ROC (receiver operating characteristic) curve. These algorithms return the correct answer when the conditional probability function learned is exact, but the effect of error in this step is not clearly understood. This method is referred to as ``ROC'' in our experimental section.

\citet{ElkanNoto08} propose another method for estimating $\kappa^*$ by constructing a conditional probability estimator which treats samples from $F$ and samples from $H$ as positive and negative classes. Even in the limit of infinite data, it is known that this estimator gives the right answer only if the supports of $G$ and $H$ are completely distinct. This method is referred to as ``EN'' in our experiments.

\citet{duPlessisSugi14} propose a method for estimating $\kappa^*$ based on Pearson divergence minimization. It can be seen as similar in spirit to the method of \citet{ElkanNoto08}, and thus has the same shortcoming of being exact only when the supports of $G$ and $H$ are disjoint, even in the limit of infinite data. The main difference between the two is that this method does not require the estimation of a conditional probability model as an intermediate object, and computes the mixture proportion directly. 

Recently, \citet{Jain+16} have proposed another method for the estimation of mixture proportion which is based on maximizing the ``likelihood'' of the mixture proportion. The algorithm suggested by them  computes a likelihood associated with each possible value of $\kappa^*$, and returns the smallest value for which the likelihood drops significantly. In a sense, it is similar to our gradient thresholding algorithm, which also computes a distance associated to each possible value of $\lambda^*$, and returns the smallest value for which the distance increases faster than a threshold. Their algorithm also requires a conditional probability model distinguishing $F$ and $H$ to be learned. It also has no guarantees of convergence to the true estimate $\kappa^*$. This method is referred to as ``alphamax'' in our experiments.

\citet{Menon+15, LiuTao15} and \citet{Scott+13b} propose to estimate the mixture proportion $\kappa^*$, based on the observation that, if the distributions $F$ and $H$ satisfy the anchor set condition, then $\kappa^*$ can be directly related to the maximum value of the conditional probability given by $\max_x \eta(x)$, where $\eta$ is the conditional probability function in the binary classification problem treating samples from $F$ as positive and samples from $H$  negative. Thus one can get an estimate of $\kappa^*$ from an estimate of the conditional probability $\hat \eta$ through $\max_x \hat\eta(x)$. This method clearly requires estimating a conditional probability model, and is also less robust to errors in estimating the conditional probability due to the form of the estimator.

\section{Experiments}
\label{sec:expts}
We ran our algorithm with 6 standard binary classification datasets\footnote{\texttt{shuttle, pageblocks, digits} are originally multiclass datasets, they are used as binary datasets by either grouping or ignoring classes.} taken from the UCI machine learning repository, the details of which are given below in Table \ref{tab:datasets}.\footnote{In our experiments, we project the data points from the \texttt{digits} and \texttt{mushroom} datasets onto a 50-dimensional space given by PCA.}

\begin{table}
\caption{Dataset statistics}
\label{tab:datasets}
\begin{center}
 \begin{tabular}{|c|c|c|c|}
  \hline 
  Dataset & \# of samples & Pos. frac. & Dim. \\ \hline
  \texttt{waveform} & 	3343	& 	0.492	&	21	\\ \hline
  \texttt{mushroom} & 	8124	& 	0.517	&	117	\\ \hline
  \texttt{pageblocks} & 5473	& 	0.897	&	10	\\ \hline
  \texttt{shuttle} & 	58000	& 	0.785	&	9	\\ \hline
  \texttt{spambase} & 	4601	& 	0.394	&	57	\\ \hline
  \texttt{digits} & 	13966	& 	0.511	&	784	\\ \hline  
 \end{tabular}
\end{center} 
\end{table}

From each binary dataset containing positive and negative labelled data points, we derived 6 different pairs of mixture and component distributions ($F$ and $H$ respectively) as follows. We chose a fraction of the positive data points to be part of the component distribution, the positive data points not chosen and the negative data points constitute the mixture distribution. The fraction of positive data points chosen to belong to the component distribution was one of $\{0.25,0.5,0.75\}$ giving 3 different pairs of distributions. The positive and negative labels were flipped and the above procedure was repeated to get 3 more pairs of distributions. From each such distribution we drew a total of either 400,800,1600 or 3200 samples and ran the two variants of our kernel mean based gradient thresholding algorithm given by ``KM1'' and ``KM2''. Our candidate kernels were five Gaussian RBF kernels, with the kernel width taking values uniformly in the log space between a tenth of the median pairwise distance and ten times the median distance, and among these kernels the kernel for which $\|\phi(\hat F)-\phi(\hat H)\|$ is highest is chosen.  We also ran the  ``alphamax'', ``EN'' and ``ROC'' algorithms for comparison.\footnote{The code for our algorithms KM1 and KM2 are at \url{http://web.eecs.umich.edu/\~cscott/code.html\#kmpe}. The code for ROC was taken from \url{http://web.eecs.umich.edu/\~cscott/code/mpe.zip}. The codes for the alphamax and EN algorithms were the same as in \citet{Jain+16}, and acquired through personal communication.} The above was repeated 5 times with different random seeds, and the average error $|\hat\kappa - \kappa^*|$ was computed. The results are plotted in Figure \ref{fig:plot_data} and the actual error values used in the plots is given in the supplementary material Section \ref{sec:expts-table}. Note that points in all plots are an average of 30 error terms arising from the 6 distributions for each dataset, and 5 different sets of samples for each distribution arising due to different random seeds. 

\begin{figure*}[tp]
\begin{center}
\begin{subfigure}[b]{0.4\textwidth} 
\includegraphics[width=\textwidth]{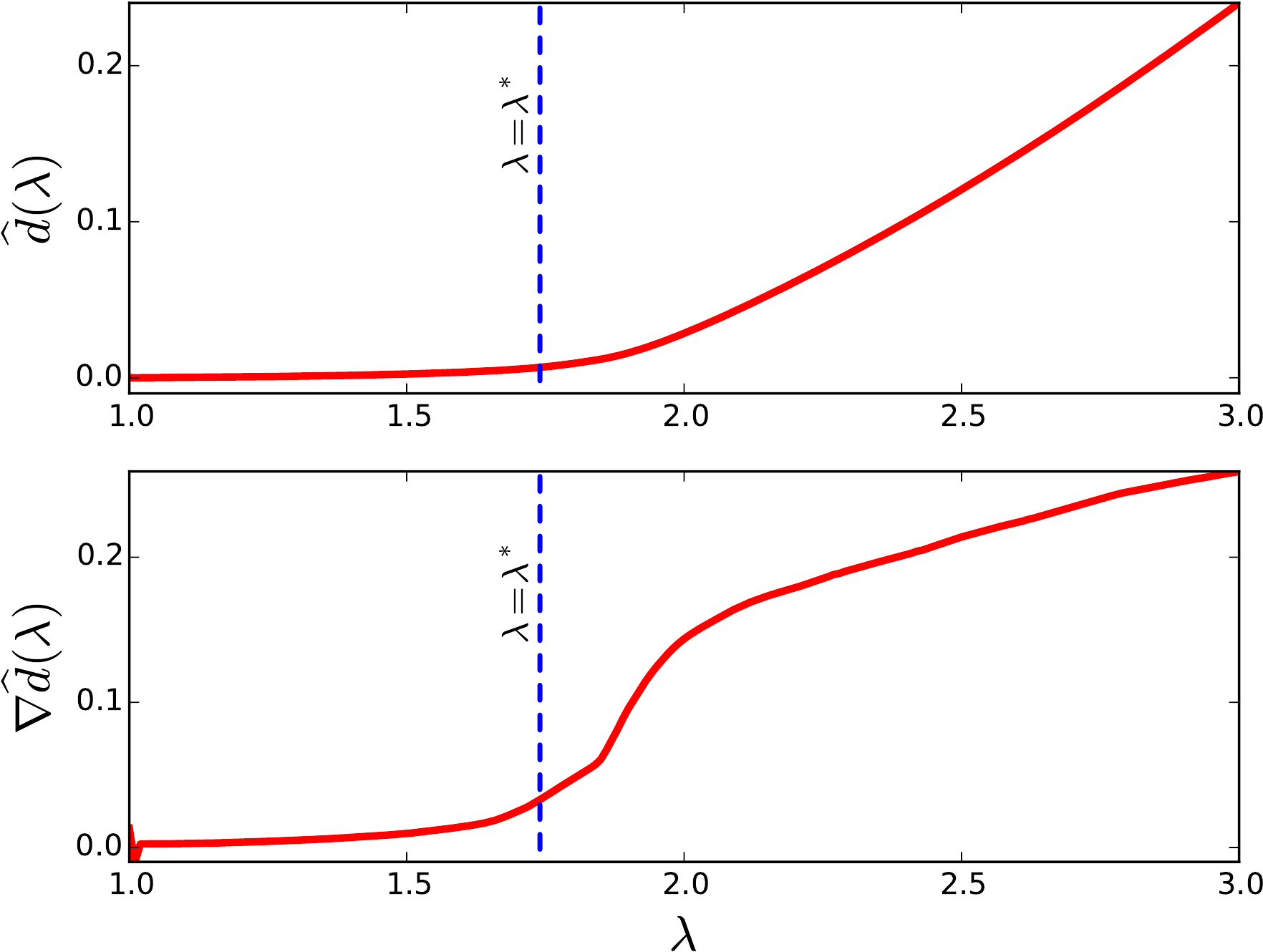}
\caption{$n+m=3200, \kappa^*=0.43, \lambda^*=1.75$}
\label{subfig:3200_25}
\end{subfigure} 
\hspace{1em}
\begin{subfigure}[b]{0.4\textwidth} 
\includegraphics[width=\textwidth]{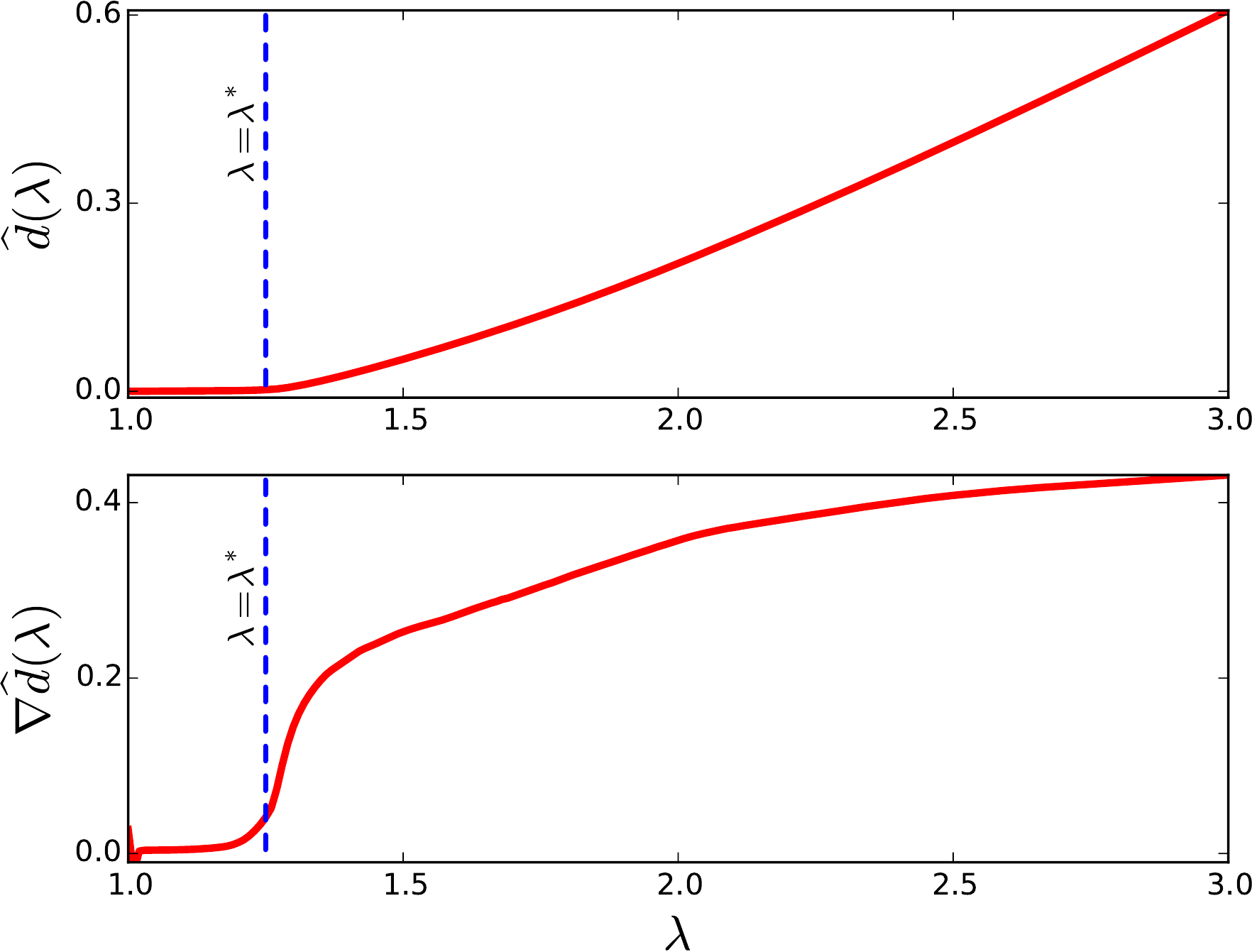}
\caption{$n+m=3200, \kappa^*=0.2, \lambda^*=1.25$}
\label{subfig:3200_75}
\end{subfigure}

\begin{subfigure}[b]{0.4\textwidth} 
\includegraphics[width=\textwidth]{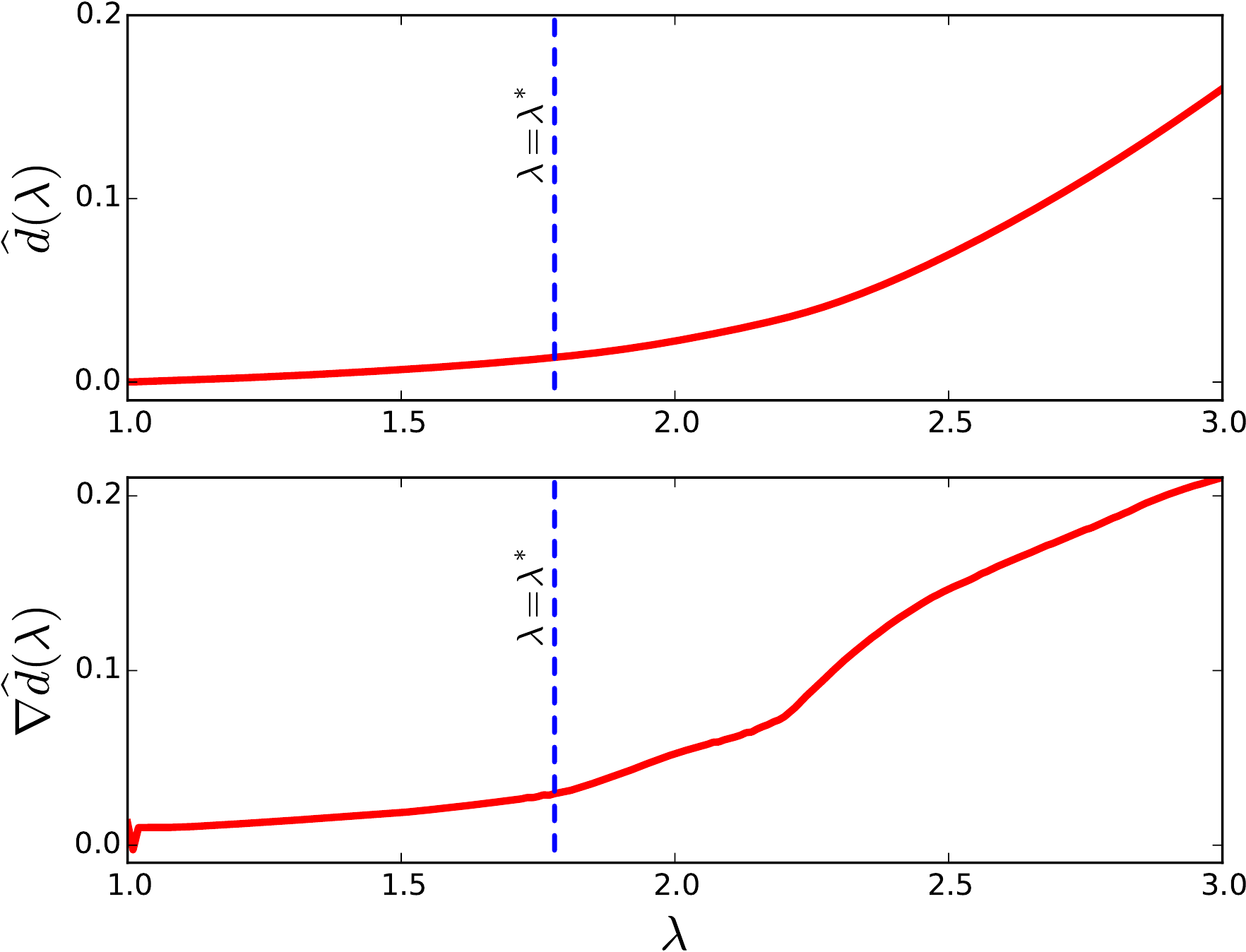}
\caption{$n+m=800, \kappa^*=0.43, \lambda^*=1.75$ }
\label{subfig:800_25}
\end{subfigure} 
\hspace{1em}
\begin{subfigure}[b]{0.4\textwidth} 
\includegraphics[width=\textwidth]{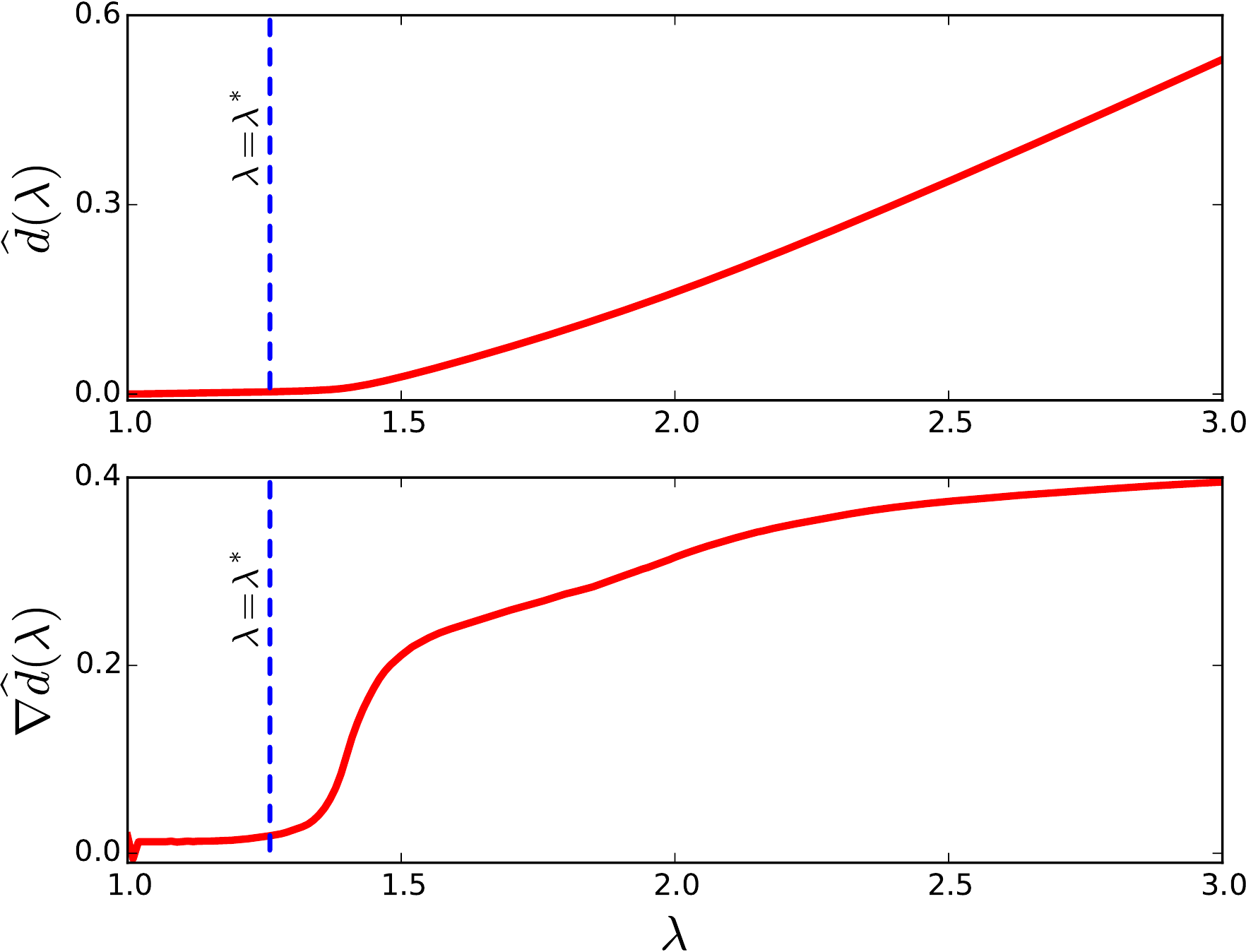}
\caption{$n+m=800, \kappa^*=0.2, \lambda^*=1.25$}
\label{subfig:800_75}
\end{subfigure}
\end{center}
\caption{$\hat d(.)$ and $\nabla\hat d(.)$ are plotted for two different sample sizes and true positive proportions.}
\label{fig:plots-illus} 
\end{figure*}
\begin{figure*}
 \begin{center}
  \includegraphics[width=0.85\textwidth]{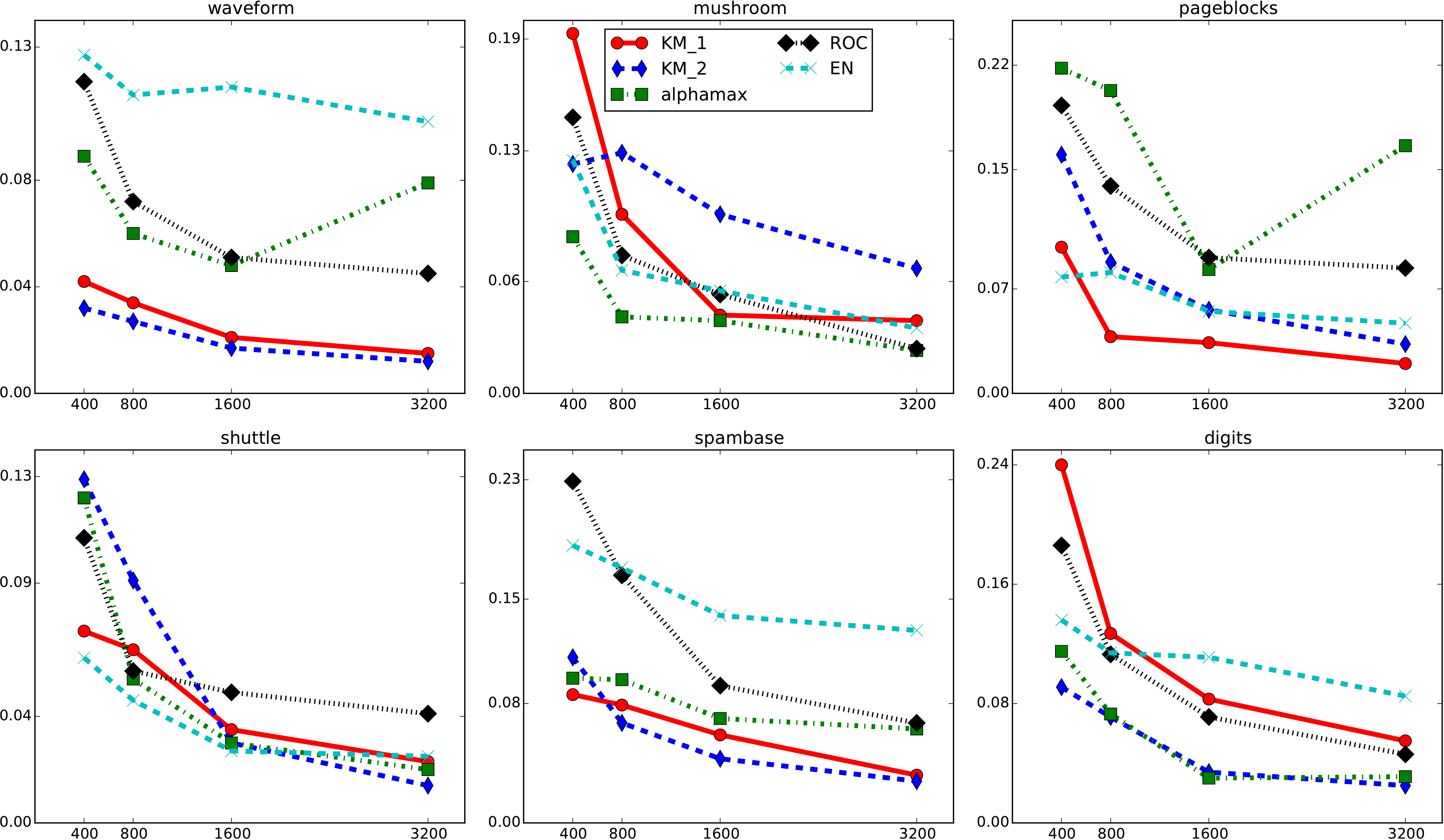}
 \end{center}
 \caption{The average error made by the KM, alphamax, ROC and EN algorithms in predicting the mixture proportion $\kappa^*$ for various datasets as a function of the total number of samples from the mixture and component.}
 \label{fig:plot_data}
\end{figure*}
It can be seen from the plots in Figure \ref{fig:plot_data}, that our algorithms (KM1 and KM2) perform comparably to or better than other algorithms for all datasets except \texttt{mushroom}.

\section{Conclusion}
\label{sec:concl}
Mixture proportion estimation is an interesting and important problem that arises naturally in many `weakly supervised learning' settings. In this paper, we give an efficient kernel mean embedding based method for this problem, and show convergence of the algorithm to the true mixture proportion under certain conditions. We also demonstrate the effectiveness of our algorithm in practice by running it on several benchmark datasets.

\section*{Acknowledgements}
This work was supported in part by NSF Grants No. 1422157, 1217880, and 1047871.

\newpage

\bibliography{kernel_MPE}
\bibliographystyle{icml2016}

\newpage

\onecolumn
\appendix
\allowdisplaybreaks
\begin{center}
\textbf{\Large Mixture Proportion Estimation 
via Kernel Embeddings of Distributions}
\\[16pt]
\textbf{\Large Supplementary Material} 
\\[16pt]
\end{center}

\section{Proof of Propositions \ref{prop:d-is-zero}, \ref{prop:d-inc-conv}, \ref{prop-d-reformulation} and \ref{prop:basic-bounds-d}}
\begin{prop*}
\begin{align*}
d(\lambda) &=0 ,  \qquad \forall \lambda\in[0,\lambda^*] , \\
\hat d(\lambda) &=0 ,  \qquad \forall \lambda\in[0,1].
\end{align*} 
\end{prop*}
\begin{proof}
The second equality is obvious and follows from convexity of $\C_S$ and that both $\phi(\hat F)$ and $\phi(\hat H)$ are in $\C_S$.

The first statement is due to the following. Let $\lambda\in[0,\lambda^*]$, then we have that,
\begin{align*}
 d(\lambda) 
 &= \inf_{w\in\C} \| \lambda \phi(F) + (1-\lambda)\phi(H) - w \|_\H\\
 &= \inf_{w\in\C} \left\| \frac{\lambda}{\lambda^*}(\lambda^* \phi(F) + (1-\lambda^*)\phi(H)) + \left(1-\frac{\lambda}{\lambda^*} \right) \phi(H)  - w \right\|_\H\\
 &= \inf_{w\in\C} \left\| \frac{\lambda}{\lambda^*}(\phi(G)) + \left(1-\frac{\lambda}{\lambda^*} \right) \phi(H)  - w \right\|_\H\\
 &= 0 \,.
\end{align*} 
\end{proof}

\begin{prop*}
 $d(.)$ and $\hat d(.)$ are non-decreasing convex functions. 
\end{prop*}
\begin{proof}
Let $0<\lambda_1<\lambda_2$. Let $\epsilon>0$. 
Let $w_1,w_2\in\C$ be such that
\begin{align*}
d(\lambda_1)&\geq \| (\lambda_1) \phi(F) + (1-\lambda_1)\phi(H) - w_1 \|_\H - \epsilon , \\ 
d(\lambda_2)&\geq \| (\lambda_2) \phi(F) + (1-\lambda_2)\phi(H) - w_2 \|_\H - \epsilon \,.
\end{align*}
By definition of $d(.)$ such $w_1,w_2$ exist for all $\epsilon>0$.

Let $\gamma\in[0,1]$,  $\lambda_\gamma= (1-\gamma)\lambda_1+\gamma\lambda_2$ and $w_\gamma=(1-\gamma)w_1+\gamma w_2$. We then have that
\begin{align*}
d(\lambda_\gamma) 
&\leq \| (\lambda_\gamma) \phi(F) + (1-\lambda_\gamma)\phi(H) - w_\gamma \|_\H \\
&= \| ((1-\gamma)\lambda_1+\gamma\lambda_2)\phi(F) + (1-(1-\gamma)\lambda_1-\gamma\lambda_2)\phi(H) - w_\gamma \|_\H \\
&= \| ((1-\gamma)\lambda_1+\gamma\lambda_2)\phi(F) + ((1-\gamma)(1-\lambda_1)+\gamma(1-\lambda_2))\phi(H) - w_\gamma \|_\H \\
&= \left\| (1-\gamma) \left( \lambda_1 \phi(F) + (1-\lambda_1)\phi(H) - w_1 \right) + \gamma \left(\lambda_2 \phi(F) + (1-\lambda_2)\phi(H) - w_2 \right) \right\| \\
&\leq (1-\gamma)\left\| \left( \lambda_1 \phi(F) + (1-\lambda_1)\phi(H) - w_1 \right)\right\| + \gamma\left\| \left(\lambda_2 \phi(F) + (1-\lambda_2)\phi(H) - w_2 \right) \right\| \\
&\leq (1-\gamma) (d(\lambda_1)+\epsilon) + \gamma (d(\lambda_2)+\epsilon)\\
&= (1-\gamma) d(\lambda_1) + \gamma d(\lambda_2) + \epsilon \, .
\end{align*}
As the above holds for all $\epsilon>0$ and $d(\lambda_\gamma)$ is independent of $\epsilon$, we have
\[ d(\lambda_\gamma) = d((1-\gamma)\lambda_1+\gamma\lambda_2) \leq (1-\gamma) d(\lambda_1) + \gamma d(\lambda_2). \]
Thus we have that $d(.)$ is convex.  

As $\C$ is convex and $\phi(H),\phi(F)\in\C$, we have that $d(\lambda)=0$ for $\lambda\in [0,\lambda^*]$, and hence $\nabla d(\lambda)=0$ for $\lambda\in[0,\lambda^*]$. By convexity, we then have that for all $\lambda\geq 0$, all elements of the sub-differential $\partial d(\lambda)$ are non-negative and hence $d(.)$ is a non-decreasing function.

By very similar arguments, we can also show that $\hat d(.)$ is convex and non-decreasing.  
\end{proof}

\begin{prop*}
For all $\mu\geq 0$
\begin{align*}
d(\lambda^*+\mu)	&= \inf_{w\in\C} \| \phi(G) + \mu(\phi(F) - \phi(H)) - w \|_\H.
\end{align*}
\end{prop*}
\begin{proof}
\begin{align*}
d(\lambda^*+\mu) &= \inf_{w\in\C} \| (\lambda^*+\mu) \phi(F) + (1-\lambda^*-\mu)\phi(H) - w \|_\H \\
&=	\inf_{w\in\C} \| \lambda^* \phi(F) + (1-\lambda^* ) \phi(H) + \mu(\phi(F)-\phi(H)) - w \|_\H \\
&=	\inf_{w\in\C} \|  \phi(\lambda^* F + (1-\lambda^* ) H) + \mu(\phi(F)-\phi(H)) - w \|_\H \,.
\end{align*}
\end{proof}

\begin{prop*}
For all $\lambda,\mu \geq 0$,
\begin{align}
d(\lambda) &\geq \lambda \|\phi(F) - \phi(H)\| - \sup_{w\in\C} \|\phi(H)-w\| , \\
d(\lambda^*+ \mu) &\leq  \mu \|\phi(F) - \phi(H)\| ,.
\end{align}
\end{prop*}
\begin{proof}
The proof of the first inequality above follows from applying triangle inequality to $d(.)$ from Equation \eqref{eqn:d-defn}.

The proof of the second inequality above follows from Proposition \ref{prop-d-reformulation} by setting $h=\phi(G)$.
\end{proof}

\section{Proof of Lemma \ref{prop:empirical-close-to-truth}}
\begin{lem*}
Let the kernel $k$ be such that $k(x,x)\leq 1$ for all $x\in\X$. Let $\delta\in(0,1/4]$. We have that, the following holds with probability $1-4\delta$ (over the sample $x_1,\ldots,x_{n+m}$) if $n>2(\lambda^*)^2 \log\left(\frac{1}{\delta}\right)$.
\begin{align*}
\|\phi(F) - \phi(\hat{F})\|_\H &\leq 
\frac{3\sqrt{\log(1/\delta)}}{\sqrt{n}} , \\
\|\phi(H) - \phi(\hat{H})\|_\H &\leq
\frac{3\sqrt{\log(1/\delta)}}{\sqrt{m}}  , \\
\|\phi(G) - \phi(\hat{G})\|_\H  &\leq
\frac{3\sqrt{\log(1/\delta)}}{\sqrt{n/(2\lambda^*)}} \,.
\end{align*}
\end{lem*}

The proof for the first two statements is a direct application of Theorem 2 of Smola et al. \cite{Smola+07}, along with bounds on the Rademacher complexity. The proof of the third statement also uses Hoeffding's inequality to show that out of the $n$ samples drawn from $F$, at least $n/(2\lambda^*)$ samples are drawn from $G$.
 
\begin{lem}
\label{lem:Rademacher-bounds}
Let the kernel $k$ be such that $k(x,x)\leq 1$ for all $x\in\X$. Then we have the following
\begin{enumerate}
\item
For all $h\in\H$ such that $\|h\|_\H \leq 1$ we have that $\sup_{x\in\X} |h(x)|\leq 1$.  
\item 
For all distributions $P$ over $\X$, the Rademacher complexity of $\H$ is bounded above as follows:
\[
R_n(\H,P)= \frac{1}{n} \E_{x_1,\ldots,x_n \sim P} \E_{\sigma_1,\ldots,\sigma_n} \left[ \sup_{h:\|h\|_\H\leq 1} \left| \sum_{i=1}^n \sigma_i h(x_i) \right| \right] \leq \frac{1}{\sqrt{n}} \,.
\] 
\end{enumerate} 
\end{lem}
\begin{proof}
The first item simply follows from Cauchy-Schwarz and the reproducing property of $\H$
\[
|h(x)| = |\langle h, k(x,.) \rangle | \leq \|h\|_\H \|k(x,.)\|_\H \leq 1 \,.
\]

The second item is also a standard result and follows from the reproducing property and Jensen's inequality.
{\allowdisplaybreaks
\begin{align*}
\frac{1}{n} \E_{\sigma_1,\ldots,\sigma_n} \left[ \sup_{h:\|h\|_\H\leq 1} \left| \sum_{i=1}^n \sigma_i h(x_i) \right| \right]
&=
\frac{1}{n} \E_{\sigma_1,\ldots,\sigma_n} \left[ \sup_{h:\|h\|_\H\leq 1} \left| \langle \sum_{i=1}^n \sigma_i k(x_i,.), h \rangle \right| \right] \\
&=
\frac{1}{n} \E_{\sigma_1,\ldots,\sigma_n} \left[  \left\|  \sum_{i=1}^n \sigma_i k(x_i,.)  \right\| \right] \\
&\leq
\frac{1}{n} \sqrt{\E_{\sigma_1,\ldots,\sigma_n} \left[  \left\|  \sum_{i=1}^n \sigma_i k(x_i,.)  \right\|^2 \right]} \\
&=
\frac{1}{n} \sqrt{\E_{\sigma_1,\ldots,\sigma_n} \left[    \sum_{i=1}^n  k(x_i,x_i)   \right] } \\
&\leq
\frac{1}{\sqrt{n}} \,.
\end{align*}
}
\end{proof}

\begin{thm}\cite{Smola+07}
\label{thm:Smola-thm}
Let $\delta\in(0,1/4]$. Let  all $h\in\H$ with $\|h\|_\H\leq 1$ be such that  $\sup_{x\in\X} |h(x)|\leq R$. Let $\hat P$ be the empirical distribution induced by $n$ i.i.d. samples from a distribution.  Then with probability at least $1-\delta$ 
\[
\|\phi(P) - \phi(\hat P)\| \leq 2 R_n(\H,P) + R \sqrt{\frac{\log\left(\frac{1}{\delta}\right)}{n}} \,.
\]
\end{thm}

\begin{lem}
\label{lem:num-samples-from-P1}
Let $\delta\in(0,1/4]$. Let $n>2(\lambda^*)^2 \log\left(\frac{1}{\delta}\right)$. Then with at least probability $1-\delta$ the following holds. At least $\frac{n}{2\lambda^*}$ of the $n$ samples $x_1,\ldots,x_n$ drawn from $F$ (which is a mixture of $G$ and $H$) are drawn from $G$.
\end{lem}
\begin{proof}
For all $1\leq i \leq n$ let 
\[
z_i=\begin{cases}
		1 &\text{ if } x_i \text{ is drawn from }G \\
		0 &\text{ otherwise}
		\end{cases} \,.
\]
From the definition of $F$, we have that $z_i$ are i.i.d. Bernoulli random variables with a bias of $\frac{1}{\lambda^*}$. Therefore by Hoeffding's inequality we have that,
\begin{align*}
Pr\left(\sum_{i=1}^n z_i > \frac{n}{2\lambda^*}\right)
&= Pr\left(\frac{1}{n}\sum_{i=1}^n z_i  -  \frac{1}{\lambda^*} > \frac{-1}{2\lambda^*}\right) \\
&= 1-Pr\left(\frac{1}{n}\sum_{i=1}^n z_i  -  \frac{1}{\lambda^*} \leq \frac{-1}{2\lambda^*}\right)\\
&\geq 1- e^{-\frac{2n}{(2\lambda^*)^2}} \geq 1-\delta \,.
\end{align*}
\end{proof}
\begin{proof}(\emph{Proof of Lemma \ref{prop:empirical-close-to-truth}})
From Theorem \ref{thm:Smola-thm} and Lemma \ref{lem:Rademacher-bounds}, we have that with probability $1-\delta$
\[
\|\phi(F) - \phi(\hat{F})\|_\H 
\leq
2 \frac{1}{\sqrt{n}} + \sqrt{\frac{\log\left(\frac{1}{\delta}\right)}{n}} \,.
\]
We also have that with probability $1-\delta$
\[
\|\phi(H) - \phi(\hat H)\|_\H 
\leq
2 \frac{1}{\sqrt{m}} + \sqrt{\frac{\log\left(\frac{1}{\delta}\right)}{m}} \,.
\]

Let $n'$ be the number of samples  in $x_1,\ldots,x_n$ drawn from $G$. From Lemma \ref{lem:num-samples-from-P1},  we have that with probability $1-\delta$ the $n' \geq \frac{n}{2\lambda^*}$. 

We also have that with probability $1-\delta$
\[
\|\phi(G) - \phi(\hat G)\|_\H 
\leq
2 \frac{1}{\sqrt{n'}} +  \sqrt{\frac{\log\left(\frac{1}{\delta}\right)}{n'}} \,.
\]

Putting the above four $1-\delta$ probability events together completes the proof.
\end{proof}

\section{Proofs of Lemmas \ref{prop:d_hat_upper} and \ref{prop:d_hat_lower}}

\begin{lem*}
Let $k(x,x)\leq 1$ for all $x\in\X$. Assume $E_\delta$. For all $\lambda\in[1,\lambda^*]$ we have that
\[
\hat d(\lambda) \leq 
\left(2-\frac{1}{\lambda^*}+\frac{\sqrt{2}}{\sqrt{\lambda^*}}\right)\lambda\cdot \frac{3\sqrt{\log(1/\delta)}}{\sqrt{\min(m,n)}}
\]
\end{lem*}
\begin{proof}
For any $\lambda\in[1,\lambda^*]$, let $w_\lambda=\frac{\lambda}{\lambda^*} \phi(\hat G) + (1-\frac{\lambda}{\lambda^*}) \phi(\hat H) \in \C_S$. 
{\allowdisplaybreaks
\begin{align*}
\hat d(\lambda)
&=
\inf_{w\in\C_S} \|\lambda \phi(\hat{F}) + (1-\lambda) \phi(\hat H) - w \|_\H \\
&\leq
\inf_{w\in\C_S} \|\lambda \phi(F) + (1-\lambda) \phi(H) -w \|_\H + (2\lambda-1)\cdot \frac{3\sqrt{\log(1/\delta)}}{\sqrt{\min(m,n)}}   \\
&=
\inf_{w\in\C_S} \| \phi(G) + (\lambda-\lambda^*) (\phi(F)-\phi(H)) -w \|_\H + (2\lambda-1)\cdot \frac{3\sqrt{\log(1/\delta)}}{\sqrt{\min(m,n)}}   \\
&=
\inf_{w\in\C_S} \left\| \phi(G) + \frac{\lambda-\lambda^*}{\lambda^*} (\phi(G)-\phi(H)) -w \right\|_\H + (2\lambda-1)\cdot \frac{3\sqrt{\log(1/\delta)}}{\sqrt{\min(m,n)}}   \\
&\leq
\left\| \phi(G) + \frac{\lambda-\lambda^*}{\lambda^*} (\phi(G)-\phi(H)) -w_\lambda \right\|_\H + (2\lambda-1)\cdot \frac{3\sqrt{\log(1/\delta)}}{\sqrt{\min(m,n)}}   \\
&=
\left\|\frac{\lambda}{\lambda^*} (\phi(G)-\phi(\hat G)) + \left(1-\frac{\lambda}{\lambda^*}\right)(\phi(H)-\phi(\hat H))\right\|_\H + (2\lambda-1)\cdot \frac{3\sqrt{\log(1/\delta)}}{\sqrt{\min(m,n)}}   \\
&\leq
\frac{\lambda}{\lambda^*}\|(\phi(G)-\phi(\hat G))\|_\H + \left(1-\frac{\lambda}{\lambda^*}\right)\|(\phi(H)-\phi(\hat H))\|_\H + (2\lambda-1)\cdot \frac{3\sqrt{\log(1/\delta)}}{\sqrt{\min(m,n)}}   \\
&\leq
\frac{\lambda}{\lambda^*} \frac{3\sqrt{\log(1/\delta)}}{\sqrt{n/(2\lambda^*)}} +
\left(1-\frac{\lambda}{\lambda^*}\right) \frac{3\sqrt{\log(1/\delta)}}{\sqrt{m}} +
(2\lambda-1)\cdot \frac{3\sqrt{\log(1/\delta)}}{\sqrt{\min(m,n)}} \\
&\leq
\frac{\lambda}{\lambda^*} \sqrt{2\lambda^*} \frac{3\sqrt{\log(1/\delta)}}{\sqrt{\min(m,n)}} +
\left(1-\frac{\lambda}{\lambda^*}\right) \frac{3\sqrt{\log(1/\delta)}}{\sqrt{\min(m,n)}} +
(2\lambda-1)\cdot \frac{3\sqrt{\log(1/\delta)}}{\sqrt{\min(m,n)}} \\
&=
\left(\frac{\sqrt {2}}{\sqrt{\lambda^*}}{\lambda}+1-\frac{\lambda}{\lambda^*} + 2\lambda -1 \right)\frac{3\sqrt{\log(1/\delta)}}{\sqrt{\min(m,n)}}\\
&=
\left(2-\frac{1}{\lambda^*}+\frac{\sqrt 2}{\sqrt {\lambda^*}}\right)\lambda\cdot \frac{3\sqrt{\log(1/\delta)}}{\sqrt{\min(m,n)}} \,.
\end{align*}
}
\end{proof}

\begin{lem*}
Let $k(x,x)\leq 1$ for all $x\in\X$. Assume $E_\delta$. For all $\lambda\geq 1$, we have
\[
\hat d(\lambda) \geq 
d(\lambda) - (2\lambda-1)\cdot \frac{3\sqrt{\log(1/\delta)}}{\sqrt{\min(m,n)}} \,.
\]
\end{lem*}
\begin{proof}
\begin{align*}
\hat d(\lambda) 
&= 
\inf_{w\in\C_S} \| \lambda \phi(\hat{F}) + (1-\lambda)\phi(\hat H) - w \|_\H \\
&\geq
\inf_{w\in\C_S} \| \lambda \phi(F) + (1-\lambda)\phi(H) - w \|_\H - \lambda \|\phi(\hat{F}) - \phi(F) \|_\H - (\lambda-1) \|\phi(H) - \phi(\hat{H})\|_\H\\
&\geq 
d(\lambda) - \lambda\cdot \frac{3\sqrt{\log(1/\delta)}}{\sqrt{n}} - (\lambda-1)\cdot\frac{3\sqrt{\log(1/\delta)}}{\sqrt{m}} \\
&\geq
d(\lambda) - (2\lambda-1)\cdot \frac{3\sqrt{\log(1/\delta)}}{\sqrt{\min(m,n)}} \,.
\end{align*}
\end{proof}

\section{Proof of Theorem \ref{thm:d_lower}}
\begin{thm*}
Let the kernel $k$, and distributions $G, H$ satisfy the separability condition with margin $\alpha>0$ and tolerance $\beta$. Then $\forall \mu>0$
\[
d(\lambda^*+\mu) \geq \frac{\alpha \mu}{\lambda^*} - \beta\,.
\]
\end{thm*}
\begin{proof} 
Let $g\in\H$ be the witness to the separability condition -- (i.e.) $\|g\|_\H\leq 1$ and $\E_{X\sim G} g(X)\leq \inf_x g(x) + \beta \leq \E_{X\sim H} g(X) + \alpha$.  Let $\Delta_\X$ denote the set of all probability distributions over $\X$. One can show that 
{
\begin{align*}
d(\lambda^*+\mu)
&= \inf_{w\in\C} \| \phi(G) + \mu(\phi(F) - \phi(H)) - w \|_\H \\
&= \inf_{P\in\Delta_\X} \| \phi(G) + \frac{\mu}{\lambda^*}(\phi(G) - \phi(H)) - \phi(P) \|_\H \\
&= \inf_{P\in\Delta_\X} \sup_{h\in\H:\|h\|\leq 1} \left\langle \phi(P) + \frac{\mu}{\lambda^*}(\phi(H) - \phi(G)) - \phi(G) , h \right\rangle \\
&= \inf_{P\in\Delta_\X} \sup_{h\in\H:\|h\|\leq 1}  \E_{P}[h(X)] - \E_G[h(X)]     + \frac{\mu}{\lambda^*}(\E_H[h(X)] - \E_G[h(X)])    \\
&\geq \inf_{P\in\Delta_\X}   \E_{P}[g(X)] + \frac{\mu}{\lambda^*}\E_H[g(X)] -\left( 1+\frac{\mu}{\lambda^*}\right) \E_G[g(X)])   \\
&\geq   \inf_x g(x) +  \frac{\mu}{\lambda^*}(\alpha) - (\inf_x g(x) + \beta)  \\
&= \frac{\alpha\mu}{\lambda^*} - \beta\,.
\end{align*}
}
\end{proof}

\section{Proof of Theorem \ref{thm:kernel-separability-universality-anchor-set}}
\begin{thm*}
Let the kernel $k:\X\times\X\>[0,\infty)$ be universal. Let the distributions $G,H$ be such that they satisfy the anchor set condition with margin $\gamma>0$ for some family of subsets of $\X$. Then, for all $\theta>0$, there exists a $\beta>0$ such that the kernel $k$, and distributions $G,H$ satisfy the separability condition with margin $\beta \theta$ and tolerance $\beta$, i.e.
\begin{align*}
 \E_{X\sim G} h(X) \leq \inf_{x} h(x) + \beta \leq \E_{X\sim H} h(X) - \beta\theta
\end{align*}
\end{thm*}
\begin{proof}
Fix some $\theta>0$. Let $A\subseteq\X$ be the witness to the anchor set condition, i.e., $A$ is a compact set such that $A\subseteq \supp(H)\setminus\supp(G)$ and $H(A)\geq \gamma$. $A$ is a compact (and hence closed) set that is disjoint from $\supp(G)$ (which is a closed, compact set), hence there exists a continuous function $f:\X\>\R$ such that, 
\begin{align*}
f(x)&\geq 0, \forall x \in \X , \\
f(x)&=0, \forall x \in \supp(G), \\
f(x)&\geq 1, \forall x \in A.
\end{align*}

By universality of the kernel $k$, we have that 
\[\forall \epsilon>0, \exists h_\epsilon \in \H, \text{ s.t. } \sup_{x\in\X} |f(x) -  h_\epsilon(x)| \leq \epsilon.\]
We then have the following:
\begin{align}
\E_{G} h_\epsilon(X) &\leq \epsilon , \label{eqn:weak-separability-universality-1} \\
\inf_{x\in\X} h_\epsilon(x) &\leq \epsilon , \label{eqn:weak-separability-universality-2} \\
\inf_{x\in\X} h_\epsilon(x) &\geq -\epsilon , \label{eqn:weak-separability-universality-3} \\
\inf_{x\in A} h_\epsilon(x) &\geq 1-\epsilon , \\ 
\E_{H} h_\epsilon(X) &\geq (-\epsilon) (1-H(A)) + (1-\epsilon) H(A)  \nonumber \\
&\geq \gamma - \epsilon . \label{eqn:weak-separability-universality-4} 
\end{align}
From Equations \eqref{eqn:weak-separability-universality-1}, \eqref{eqn:weak-separability-universality-2}, \eqref{eqn:weak-separability-universality-3} and \eqref{eqn:weak-separability-universality-4}, we have that
\[
\E_{G} h_\epsilon(X) \leq \epsilon \leq \inf_x h_\epsilon(x) +  2\epsilon \leq 3\epsilon \leq \E_{H} h_\epsilon(X) - (\gamma-4\epsilon).
\]
Let $\overline h_\epsilon= h_\epsilon/\|h_\epsilon\|_\H$ be the normalized version of $h_\epsilon$. We then have that
\[
\E_{G} \overline h_\epsilon(X) \leq \inf_x \overline h_\epsilon(x) + \frac{2\epsilon}{\|h_\epsilon\|_\H} \leq \E_{H} \overline h_\epsilon(X) - \frac{\gamma-4\epsilon}{\|h_\epsilon\|_\H}.
\]
Setting $\epsilon=\frac{\gamma}{2\theta+4}$ and $\beta=\frac{2\gamma}{(2\theta+4)\|h_{\gamma/(2\theta+4)}\|_\H}$ we get that there exists $h\in\H$ such that $\|h\|_\H \leq 1$ and 
\[
\E_{G} h(X) \leq \inf_x h(x) + \beta \leq \E_{H}  h(X) - \beta \theta.
\]
\end{proof}

\section{Proof of Theorem \ref{thm:val-thresh-error}}
\begin{thm*}
Let $\delta\in(0,\frac{1}{4}]$. Let $k(x,x)\leq 1$ for all $x\in\X$. Let the kernel $k$, and distributions $G, H$ satisfy the separability condition with tolerance $\beta$ and margin $\alpha>0$. Let the number of samples be large enough such that  $\min(m,n)>\frac{(12  \cdot \lambda^*)^2\log(1/\delta)}{\alpha^2}$. Let the threshold $\tau$ be such that $\frac{3\lambda^*\sqrt{\log(1/\delta)}(2-1/\lambda^*+\sqrt{2/\lambda^*})}{\sqrt{\min(m,n)}} \leq \tau \leq \frac{6\lambda^*\sqrt{\log(1/\delta)}(2-1/\lambda^*+\sqrt{2/\lambda^*})}{\sqrt{\min(m,n)}} $. We then have with probability $1-4\delta$
\begin{align*}
   \lambda^* - \hat \lambda^V_\tau 
   &\leq  0 ,\\
     \hat \lambda^V_\tau -   \lambda^* 
   &\leq  \frac{\beta\lambda^*}{\alpha} + c  \cdot \sqrt{\log(1/\delta)} \cdot (\min(m,n))^{-1/2} ,
\end{align*}
for constant $c=\left( \frac{6\alpha (\lambda^*)^2(2-1/\lambda^*+\sqrt{2/\lambda^*})+  2\lambda^*(3\alpha+6\lambda^*(2+\alpha+\beta))  }{\alpha^2} \right)$.
\end{thm*}

\begin{lem}
\label{lem:feasible-lambda-d-hat}
Let $k(x,x)\leq 1$ for all $x\in\X$. Let the kernel $k$, and distributions $G, H$ satisfy the  separability condition with margin $\alpha$ and tolerance $\beta$. Assume $E_\delta$. Then
\begin{align*}
 \hat d(\lambda) &\leq 
\left(2-\frac{1}{\lambda^*}+\frac{\sqrt 2}{\sqrt {\lambda^*}}\right)\lambda\cdot \frac{3\sqrt{\log(1/\delta)}}{\sqrt{\min(m,n)}}  , \qquad \forall \lambda\in[1,\lambda^*] , \\
\hat d(\lambda) &\geq 
\frac{(\lambda - \lambda^*) \alpha}{\lambda^*} - \beta -  (2\lambda-1)\cdot \frac{3\sqrt{\log(1/\delta)}}{\sqrt{\min(m,n)}} , \qquad \forall \lambda\in[\lambda^*,\infty) \,.
\end{align*}
\end{lem}
\begin{proof}
The proof follows from Lemmas \ref{prop:d_hat_lower}, \ref{prop:d_hat_upper} and Theorem \ref{thm:d_lower}. The upper bound forms the line $(\lambda,U(\lambda))$ and the lower bound forms the line $(\lambda,L(\lambda))$ in Figure \ref{subfig:distance-lambda}.
\end{proof}

\begin{lem}
\label{lem:value-thres-estimator-bounds}
Let $k(x,x)\leq 1$ for all $x\in\X$. Let the kernel $k$, and distributions $G, H$ satisfy the separability condition with margin $\alpha$ and tolerance $\beta$. Assume $E_\delta$. We then have
\begin{align}
 \hat \lambda^V_\tau 
 &\geq \min \left(\lambda^*,  \frac{\tau\sqrt{\min(m,n)}}{3\sqrt{\log(1/\delta)}(2-1/\lambda^*+\sqrt{2/\lambda^*})} \right) , \label{eqn:eq5} \\
 \hat \lambda^V_\tau 
 &\leq
\lambda^* \cdot \frac{(\tau+\beta+\alpha)\sqrt{\min(m,n)}+3\sqrt{\log(1/\delta)}}{\alpha \sqrt{\min(m,n)}-6\lambda^*\sqrt{\log(1/\delta)}} \,. \label{eqn:eq6}
\end{align}
\end{lem}
\begin{proof}
As $\hat d$ is a continuous function, we have that $\hat d(\hat \lambda^V_\tau)=\tau$. If $\hat\lambda^V_\tau \leq \lambda^*$, we have from Lemma \ref{lem:feasible-lambda-d-hat} that
\[
\tau
\leq 
\left(2-\frac{1}{\lambda^*}+\frac{\sqrt 2}{\sqrt {\lambda^*}}\right)\hat \lambda^V_\tau\cdot \frac{3\sqrt{\log(1/\delta)}}{\sqrt{\min(m,n)}},
\]
and hence
\[
\hat \lambda^V_\tau 
\geq 
\min \left(\lambda^*,  \frac{\tau\sqrt{\min(m,n)}}{3\sqrt{\log(1/\delta)}(2-1/\lambda^*+\sqrt{2/\lambda^*})} \right) \,.
\]
If $\hat\lambda^V_\tau \geq \lambda^*$, we have
\begin{align*}
\tau
&\geq 
\frac{(\hat\lambda^V_\tau - \lambda^*) \alpha}{\lambda^*} - \beta -  (2\hat\lambda^V_\tau-1)\cdot \frac{3\sqrt{\log(1/\delta)}}{\sqrt{\min(m,n)}},\\
&=\hat\lambda^V_\tau
\left(\frac{\alpha}{\lambda^*} - \frac{6\sqrt{\log(1/\delta)}}{\sqrt{\min(m,n)}}\right) 
-\alpha-\beta-\frac{3\sqrt{\log(1/\delta)}}{\sqrt{\min(m,n)}} \,.
\end{align*}
Rearranging terms, we have that if  $\hat\lambda^V_\tau \geq \lambda^*$, then
\[
\hat\lambda^V_\tau \leq \frac{\tau+\alpha+\beta+\frac{3\sqrt{\log(1/\delta)}}{\sqrt{\min(m,n)}}}
{\frac{\alpha}{\lambda^*} - \frac{6\sqrt{\log(1/\delta)}}{\sqrt{\min(m,n)}}}\,.
\]
And thus
\begin{align*}
\hat\lambda^V_\tau 
&\leq \max\left(\lambda^*, \lambda^* \cdot \frac{(\tau+\beta+\alpha)\sqrt{\min(m,n)}+3\sqrt{\log(1/\delta)}}{\alpha \sqrt{\min(m,n)}-6\lambda^*\sqrt{\log(1/\delta)}}\right) \\
&=\lambda^* \cdot \frac{(\tau+\beta+\alpha)\sqrt{\min(m,n)}+3\sqrt{\log(1/\delta)}}{\alpha \sqrt{\min(m,n)}-6\lambda^*\sqrt{\log(1/\delta)}}\,.
\end{align*}
\end{proof}

\begin{proof}(Proof of Theorem \ref{thm:val-thresh-error})

As $\min(m,n)>\frac{(12  \cdot \lambda^*)^2\log(1/\delta)}{\alpha^2}>2(\lambda^*)^2\log(1/\delta)$, we have that $E_\delta$ is $1-4\delta$ probability event.
Assume $E_\delta$.

As $\tau \geq \frac{3\lambda^*\sqrt{\log(1/\delta)}(2-1/\lambda^*+\sqrt{2/\lambda^*})}{\sqrt{\min(m,n)}} $, we have from Equation \eqref{eqn:eq5}
\begin{align*}
\hat \lambda^V_\tau \geq \lambda^* \,.
\end{align*}

From Equation \eqref{eqn:eq6}, we have
\begin{align*}
\hat\lambda^V_\tau
&\leq \lambda^* \cdot \frac{(\tau+\alpha+\beta)\sqrt{\min(m,n)}+3\sqrt{\log(1/\delta)}}{\alpha \sqrt{\min(m,n)}-6\lambda^*\sqrt{\log(1/\delta)}}  \\
&=\lambda^* \left( \frac{\tau+\beta + \alpha }{\alpha} +  \frac{(3+\frac{6\lambda^*(\tau+\alpha+\beta)}{\alpha})\sqrt{\log(1/\delta)}}{\alpha \sqrt{\min(m,n)}-6\lambda^*\sqrt{\log(1/\delta)}} \right) \\
&\leq \lambda^* \left( 1 + \frac{\beta}{\alpha}\right) + \frac{\tau\lambda^*}{\alpha} +  \frac{2\lambda^*(3+\frac{6\lambda^*(\tau+\alpha+\beta)}{\alpha})\sqrt{\log(1/\delta)}}{\alpha \sqrt{\min(m,n)}}  \\
&\leq \lambda^*\left( 1 + \frac{\beta}{\alpha}\right)  + \frac{6(\lambda^*)^2\sqrt{\log(1/\delta)}(2-1/\lambda^*+\sqrt{2/\lambda^*}) }{\alpha \sqrt{\min(m,n)}}  +  \frac{2\lambda^*(3\alpha+6\lambda^*(\tau+\alpha+\beta))\sqrt{\log(1/\delta)}}{\alpha^2 \sqrt{\min(m,n)}}  \\
&\leq \lambda^*\left( 1 + \frac{\beta}{\alpha}\right)  + \left( \frac{6\alpha (\lambda^*)^2(2-1/\lambda^*+\sqrt{2/\lambda^*})+  2\lambda^*(3\alpha+6\lambda^*(2+\alpha+\beta))  }{\alpha^2} \right) \cdot \sqrt{\log(1/\delta)} \cdot (\min(m,n))^{-1/2} \,. \\
\end{align*}
The third line above follows, because  $\min(m,n)>\frac{(12  \cdot \lambda^*)^2\log(1/\delta)}{\alpha^2}$. The last two lines follow, because $\tau \leq \frac{6\lambda^*\sqrt{\log(1/\delta)}(2-1/\lambda^*+\sqrt{2/\lambda^*})}{\sqrt{\min(m,n)}}$, which in turn is upper bounded by $2$ under the conditions on $\min(m,n)$.

\end{proof}

\section{Proof of Theorem \ref{thm:grad-thresh-error}}
\begin{thm*}
Let $k(x,x)\leq 1$ for all $x\in\X$. Let the kernel $k$, and distributions $G, H$ satisfy the separability condition with tolerance $\beta$ and margin $\alpha>0$. Let $\nu\in[\frac{\alpha}{4\lambda^*},  \frac{3 \alpha}{4 \lambda^*}]$ and $\sqrt{\min(m,n)}\geq\frac{36 \sqrt{\log(1/\delta)}}{\frac{\alpha}{\lambda^*}-\nu}$. We then have with probability $1-4\delta$,
\begin{align*}
   \lambda^* - \hat \lambda^G_\nu 
   &\leq  c  \cdot \sqrt{\log(1/\delta)} \cdot (\min(m,n))^{-1/2} , \\
     \hat \lambda^G_\nu -  \lambda^* 
   &\leq  \frac{4\beta\lambda^*}{\alpha} + c'  \cdot \sqrt{\log(1/\delta)} \cdot (\min(m,n))^{-1/2} ,
\end{align*}
for constants $c=(2\lambda^*-1+\sqrt{2\lambda^*}) \cdot \frac{12\lambda^*}{\alpha}$ and $c'=\frac{144 (\lambda^*)^2(\alpha+4\beta)}{\alpha^2}$.
\end{thm*}

\begin{lem}
\label{lem:feasible-lambda-grad-d-hat}
Let $k(x,x)\leq 1$ for all $x\in\X$. Let the kernel $k$, and distributions $G, H$ satisfy the  separability condition with margin $\alpha$ and tolerance $\beta$. Assume $E_\delta$. We then have
\begin{align*}
\sup \{g\in \partial \hat d(\lambda) \} &\leq 
\frac{1}{\lambda^*-\lambda} \cdot (2\lambda^*-1+\sqrt {2\lambda^*})\cdot \frac{3\sqrt{\log(1/\delta)}}{\sqrt{\min(m,n)}},  \qquad \forall \lambda\in[1,\lambda^*]  , \\
\inf \{g\in  \partial\hat d(\lambda)\} &\geq 
\left(\frac{ \alpha}{\lambda^*} -  \frac{\beta}{\lambda-\lambda^*}- \frac{6\lambda}{\lambda-\lambda^*}\cdot \frac{3\sqrt{\log(1/\delta)}}{\sqrt{\min(m,n)}}\right),  \qquad \forall \lambda\in[\lambda^*,\infty) \,.
\end{align*}
\end{lem}
\begin{proof}

As $\hat d(.)$ is convex, we have that for all $\lambda \in [1,\lambda^*]$, and all $g\in\partial \hat d(\lambda)$
\begin{align*}
g
&\leq
\frac{\hat d(\lambda^*)  -  \hat d(\lambda) }{\lambda^* - \lambda} \\
&\leq 
\frac{\hat d(\lambda^*)  }{\lambda^* - \lambda} \,.
\end{align*}
Applying Lemma \ref{prop:d_hat_upper} to $\hat d(\lambda^*)$, we get $\forall \lambda\in[1,\lambda^*]$
\[
\sup\{g\in\partial \hat d(\lambda)\} \leq 
\frac{1}{\lambda^*-\lambda} \cdot (2\lambda^*-1+\sqrt {2\lambda^*})\cdot \frac{3\sqrt{\log(1/\delta)}}{\sqrt{\min(m,n)}} \,.
\]

Once again by convexity of $\hat d(.)$, we have that for all $\lambda\geq \lambda^*$ and all $g\in\partial \hat d(\lambda)$
\begin{align*}
g
&\geq \frac{\hat d(\lambda) - \hat d(\lambda^*)}{\lambda - \lambda^*} \,.  
\end{align*}
Applying Lemma \ref{prop:d_hat_lower} and Theorem \ref{thm:d_lower} to $\hat d(\lambda)$ and Lemma \ref{prop:d_hat_upper} to $\hat d(\lambda^*)$, we get $\forall \lambda\in[\lambda^*,\infty)$ 
\begin{align*}
\inf\{g\in\partial \hat d(\lambda)\}
&\geq
\left(\frac{ \alpha}{\lambda^*} - \frac{\beta}{\lambda-\lambda^*} - \frac{2\lambda + 2\lambda^* -2  +\sqrt{2\lambda^*}}{\lambda-\lambda^*}\cdot \frac{3\sqrt{\log(1/\delta)}}{\sqrt{\min(m,n)}}\right)\,.
\end{align*}
\end{proof}

\begin{lem}
\label{lem:grad-thres-estimator-bounds}
Let $k(x,x)\leq 1$ for all $x\in\X$. Let the kernel $k$, and distributions $G, H$ satisfy the  separability condition with margin $\alpha$ and tolerance $\beta$. Assume $E_\delta$. We then have
\begin{align}
\hat \lambda^G_\nu  
&\geq
\lambda^* -  (2\lambda^*-1+\sqrt{2\lambda^*})\cdot \frac{3\sqrt{\log(1/\delta)}}{\nu \sqrt{\min(m,n)}} ,
\label{eqn:eq3} \\
\hat \lambda^G_\nu 
&\leq  
\lambda^*\cdot \frac{\frac{ \alpha+\beta}{\lambda^*} - \nu  }
{\frac{ \alpha}{\lambda^*} - \nu - \frac{18 \sqrt{\log(1/\delta)}}{\sqrt{\min(m,n)}}} \,. \label{eqn:eq4}
\end{align}
\end{lem}
\begin{proof}
By definition of the gradient thresholding estimator $\hat \lambda^G_\nu $ we have
\[
\inf \{ g \in \partial \hat d(\hat \lambda^G_\nu) \}
\leq 
\nu
\leq
\sup \{ g \in \partial \hat d(\hat \lambda^G_\nu) \} \,.
\]
Firstly, note that $\hat\lambda^G_\nu \geq 1$, because $\nu \geq \frac{\alpha}{4\lambda^*}>0$. 
By Lemma \ref{lem:feasible-lambda-grad-d-hat} we have that if $\hat\lambda^G_\nu \in [1,\lambda^*]$ then
\begin{align}
\nu \leq 
\sup \{g\in \partial \hat d(\hat \lambda^G_\nu) \} &\leq 
\frac{1}{\lambda^*-\hat \lambda^G_\nu} \cdot (2\lambda^*-1+\sqrt {2\lambda^*})\cdot \frac{3\sqrt{\log(1/\delta)}}{\sqrt{\min(m,n)}} \,. \label{eqn:grad-threshold-est-lower-bound} 
\end{align}
Once again by Lemma \ref{lem:feasible-lambda-grad-d-hat}, we have that if $\hat\lambda^G_\nu > \lambda^*$ then
\begin{align}
\nu \geq
\inf \{g\in  \partial\hat d(\hat \lambda^G_\nu)\} &\geq 
\left(\frac{ \alpha}{\lambda^*} -  \frac{\beta}{\hat \lambda^G_\nu-\lambda^*}- \frac{6\hat \lambda^G_\nu}{\hat \lambda^G_\nu-\lambda^*}\cdot \frac{3\sqrt{\log(1/\delta)}}{\sqrt{\min(m,n)}}\right)\,. \label{eqn:grad-threshold-est-upper-bound}
\end{align}
Rearranging Equation \eqref{eqn:grad-threshold-est-lower-bound}, we get that if $\hat\lambda^G_\nu \in [1,\lambda^*]$ then
\[
\hat \lambda^G_\nu  
\geq
\lambda^* -  (2\lambda^*-1+\sqrt{2\lambda^*})\cdot \frac{3\sqrt{\log(1/\delta)}}{\nu \sqrt{\min(m,n)}} \,.
\]
Hence 
\[
\hat \lambda^G_\nu  
\geq
\min\left(\lambda^*,\lambda^* -  (2\lambda^*-1+\sqrt{2\lambda^*})\cdot \frac{3\sqrt{\log(1/\delta)}}{\nu \sqrt{\min(m,n)}}\right) 
=
\lambda^* -  (2\lambda^*-1+\sqrt{2\lambda^*})\cdot \frac{3\sqrt{\log(1/\delta)}}{\nu \sqrt{\min(m,n)}} \,.
\]
Rearranging Equation \eqref{eqn:grad-threshold-est-upper-bound}, we get that if $\hat\lambda^G_\nu >\lambda^*$, then
\begin{align*}
\frac{ \alpha}{\lambda^*} - \nu  
&\leq 
\left( \frac{6\hat \lambda^G_\nu}{\hat \lambda^G_\nu-\lambda^*}\cdot \frac{3\sqrt{\log(1/\delta)}}{\sqrt{\min(m,n)}} + \frac{\beta}{\hat \lambda^G_\nu-\lambda^*}\right) \\
\left(\hat \lambda^G_\nu-\lambda^*\right)\left(\frac{ \alpha}{\lambda^*} - \nu  \right)
&\leq
\left( 6\hat \lambda^G_\nu \right)\cdot \frac{3\sqrt{\log(1/\delta)}}{\sqrt{\min(m,n)}} + \beta \\
\left(\hat \lambda^G_\nu\right)\left(\frac{ \alpha}{\lambda^*} - \nu - \frac{18 \sqrt{\log(1/\delta)}}{\sqrt{\min(m,n)}} \right)
&\leq
\lambda^*\left(\frac{ \alpha}{\lambda^*} - \nu \right) + \beta  \\
\hat \lambda^G_\nu 
&\leq
\frac{\lambda^*\left(\frac{ \alpha+\beta}{\lambda^*} - \nu \right)  }
{\frac{ \alpha}{\lambda^*} - \nu - \frac{18 \sqrt{\log(1/\delta)}}{\sqrt{\min(m,n)}}} \,. 
\end{align*}
Thus we have
\[
\hat \lambda^G_\nu  
\leq
\max\left(\lambda^*,\frac{\lambda^*\left(\frac{ \alpha+\beta}{\lambda^*} - \nu \right)  }
{\frac{ \alpha}{\lambda^*} - \nu - \frac{18 \sqrt{\log(1/\delta)}}{\sqrt{\min(m,n)}}} \right)
=
\frac{\lambda^*\left(\frac{ \alpha+\beta}{\lambda^*} - \nu \right)  }
{\frac{ \alpha}{\lambda^*} - \nu - \frac{18 \sqrt{\log(1/\delta)}}{\sqrt{\min(m,n)}}} \,.
\]
\end{proof}

\begin{proof} (Proof of Theorem \ref{thm:grad-thresh-error})

As $(\frac{ \alpha}{\lambda^*} - \nu)\sqrt{\min(m,n)} \geq 36 \sqrt{\log(1/\delta)}$, we have that 
\[
\min(m,n)\geq \frac{(36\lambda^*)^2 \log(1/\delta)}{\alpha^2} \geq 
2(\lambda^*)^2\log(1/\delta), \]
and hence $E_\delta$ is a $1-4\delta$ probability event.
Assume $E_\delta$.

 Equation \eqref{eqn:eq3} immediately gives
\begin{align*}
\lambda^* - \hat \lambda^G_\nu 
&\leq 
(2\lambda^*-1+\sqrt{2\lambda^*})\cdot \frac{3}{\nu} \cdot \sqrt{\log(1/\delta)} \cdot (\min(m,n))^{-1/2} \\
&\leq
(2\lambda^*-1+\sqrt{2\lambda^*}) \cdot \frac{12\lambda^*}{\alpha} \cdot \sqrt{\log(1/\delta)} \cdot (\min(m,n)^{-1/2} \,.
\end{align*}
The second inequality above is due to $\nu \geq \frac{\alpha}{4\lambda^*}$.

Let $\omega= \frac{\alpha+\beta-\nu\lambda^*}{\alpha-\nu\lambda^*} \leq 1+\frac{4\beta}{\alpha}$. Equation \eqref{eqn:eq4}  gives
\begin{align*}
\hat \lambda^G_\nu
&\leq 
\lambda^*\cdot \frac{ \frac{\alpha+\beta}{\lambda^*} - \nu  }
{\frac{\alpha}{\lambda^*} - \nu - \frac{18  \sqrt{\log(1/\delta)}}{\sqrt{\min(m,n)}}} \\
&=\lambda^*\cdot \frac{ \omega \left(\frac{\alpha}{\lambda^*} - \nu-\frac{18  \sqrt{\log(1/\delta)}}{\sqrt{\min(m,n)}} \right) + \omega\left(\frac{18  \sqrt{\log(1/\delta)}}{\sqrt{\min(m,n)}}\right)    }
{\frac{\alpha}{\lambda^*} - \nu - \frac{18  \sqrt{\log(1/\delta)}}{\sqrt{\min(m,n)}}} \\
&= \omega \lambda^* + \frac{18 \omega \lambda^* \sqrt{\log(1/\delta)}}
{(\frac{ \alpha}{\lambda^*} - \nu)\sqrt{\min(m,n)} - 18 \sqrt{\log(1/\delta)}} \\
&\leq 
\omega \lambda^* + \frac{36 \omega \lambda^* \sqrt{\log(1/\delta)}}
{(\frac{ \alpha}{\lambda^*} - \nu)\sqrt{\min(m,n)} } \\
&\leq
\lambda^* + \frac{4\beta\lambda^*}{\alpha} + \frac{36(1+\frac{4\beta}{\alpha}) \lambda^* \sqrt{\log(1/\delta)}}
{(\frac{ \alpha}{4\lambda^*} )\sqrt{\min(m,n)} } \\
&\leq
\lambda^* + \frac{4\beta\lambda^*}{\alpha} + \frac{144 (\lambda^*)^2(\alpha+4\beta)}{\alpha^2} \cdot  \sqrt{\log(1/\delta)} \cdot (\min(m,n))^{-1/2} \,.\\
\end{align*}
The second inequality above is due to 
$(\frac{ \alpha}{\lambda^*} - \nu)\sqrt{\min(m,n)} \geq 36 \sqrt{\log(1/\delta)}.$ The third inequality above is due to $\nu \leq \frac{3\alpha}{4\lambda^*}$.
\end{proof}

\section{Experimental Results in Table Format} 
\label{sec:expts-table}
\begin{table*}[h]
\begin{center}
\begin{tabular}{|c|c|c|c|c|c|}
\hline
					& KM1 & KM2 & alphamax & ROC & EN \\ \hline
\texttt{waveform}(400) & 0.042 & \textbf{0.032} & 0.089${}^*$ & 0.117${}^*$ & 0.127${}^*$\\ \hline
\texttt{waveform}(800) & 0.034 & \textbf{0.027} & 0.06${}^*$ & 0.072 ${}^*$ & 0.112${}^*$\\ \hline
\texttt{waveform}(1600) & 0.021 & \textbf{0.017} & 0.048${}^*$ & 0.051${}^*$ & 0.115${}^*$\\ \hline
\texttt{waveform}(3200) & 0.015 & \textbf{0.012} & 0.079${}^*$ & 0.045${}^*$ & 0.102${}^*$\\ \hline
\texttt{mushroom}(400) & 0.193${}^*$ & 0.123${}^*$ & \textbf{0.084} & 0.148${}^*$ & 0.125\\ \hline
\texttt{mushroom}(800) & 0.096${}^*$ & 0.129${}^*$ & \textbf{0.041} & 0.074${}^*$ & 0.066${}^*$\\ \hline
\texttt{mushroom}(1600) & 0.042 & 0.096${}^*$ & \textbf{0.039} & 0.053${}^*$ & 0.055${}^*$\\ \hline
\texttt{mushroom}(3200) & 0.039${}^*$ & 0.067${}^*$ & \textbf{0.023} & 0.024 & 0.035${}^*$\\ \hline
\texttt{pageblocks}(400) & 0.098 & 0.16 & 0.218 & 0.193${}^*$ & \textbf{0.078}\\ \hline
\texttt{pageblocks}(800) & \textbf{0.038} & 0.088${}^*$ & 0.203${}^*$ & 0.139${}^*$ & 0.081${}^*$\\ \hline
\texttt{pageblocks}(1600) & \textbf{0.034} & 0.056${}^*$ & 0.083${}^*$ & 0.091${}^*$ & 0.055${}^*$\\ \hline
\texttt{pageblocks}(3200) & \textbf{0.02} & 0.033${}^*$ & 0.166${}^*$ & 0.084${}^*$ & 0.047${}^*$\\ \hline
\texttt{shuttle}(400) & 0.072 & 0.129 & 0.122 & 0.107${}^*$ & \textbf{0.062}\\ \hline
\texttt{shuttle}(800) & 0.065 & 0.091 & 0.054 & 0.057 & \textbf{0.046}\\ \hline
\texttt{shuttle}(1600) & 0.035 & 0.03 & 0.03 & 0.049${}^*$ & \textbf{0.027}\\ \hline
\texttt{shuttle}(3200) & 0.023${}^*$ & \textbf{0.014} & 0.02 & 0.041${}^*$ & 0.025${}^*$\\ \hline
\texttt{spambase}(400) & \textbf{0.086} & 0.111 & 0.097 & 0.229${}^*$ & 0.186${}^*$\\ \hline
\texttt{spambase}(800) & 0.079 & \textbf{0.067} & 0.096${}^*$ & 0.166${}^*$ & 0.171${}^*$\\ \hline
\texttt{spambase}(1600) & 0.059 & \textbf{0.043} & 0.07${}^*$ & 0.092${}^*$ & 0.139${}^*$\\ \hline
\texttt{spambase}(3200) & 0.032 & \textbf{0.028} & 0.063${}^*$ & 0.067${}^*$ & 0.129${}^*$\\ \hline
\texttt{digits}(400) & 0.24${}^*$ & \textbf{0.091} & 0.115 & 0.186${}^*$ & 0.136\\ \hline
\texttt{digits}(800) & 0.127${}^*$ & \textbf{0.07}1 & 0.073 & 0.113${}^*$ & 0.114${}^*$\\ \hline
\texttt{digits}(1600) & 0.083${}^*$ & 0.034 & \textbf{0.03} & 0.071${}^*$ & 0.111${}^*$\\ \hline
\texttt{digits}(3200) & 0.055${}^*$ & \textbf{0.025} & 0.031 & 0.046${}^*$ & 0.085${}^*$\\ \hline
\end{tabular}
\end{center}
\caption{Average absolute error incurred in predicting the mixture proportion $\kappa^*$. The first column gives the dataset and the total number of samples used (mixture and component) in parantheses. The best performing algorithm for each dataset and sample size is highlighted in bold. Algorithms whose performances have been identified as significantly inferior to the best algorithm, by the Wilcoxon signed rank test (at significance level $p=0.05$), are marked with a star. }
\end{table*}

\end{document}